%% file: main.tex
\documentclass[11pt,centertags,oneside]{article}
\RequirePackage[utf8]{inputenc}
\UseRawInputEncoding
\usepackage{amstext,amsthm,amscd,typearea,hyperref}
\usepackage{amsmath}
\usepackage{amssymb}
\usepackage{a4wide}
\usepackage[mathscr]{eucal}
\usepackage{mathrsfs}
\usepackage{typearea}
\usepackage{charter}
\usepackage{pdfsync}
\usepackage[a4paper,width=16.2cm,top=3cm,bottom=3cm]{geometry}
\usepackage{algorithm}
\usepackage{algpseudocode}
\usepackage{braket}
\usepackage{physics}
\usepackage{comment}
\usepackage{multicol}
\usepackage{graphicx}

\numberwithin{equation}{section}

\usepackage[displaymath, mathlines, right, pagewise]{lineno}

\usepackage{amsmath,amsthm,amsfonts,amssymb,mathrsfs,bm,bbm}
\usepackage[usenames]{color}
\usepackage{hyperref}
\usepackage{amssymb}
\usepackage{appendix}

\newtheorem{theorem}{Theorem}[section]
\newtheorem{definition}[theorem]{Definition}
\newtheorem{proposition}[theorem]{Proposition}
\newtheorem{corollary}[theorem]{Corollary}
\newtheorem{lemma}[theorem]{Lemma}
\newtheorem{remark}[theorem]{Remark}
\newtheorem{example}[theorem]{Example}

\newcommand{\Var}{{\mathop{\mathrm{Var}}\nolimits}}

\newcommand{\B}{\mathbb{B}}
\newcommand{\C}{\mathbb{C}}

\newcommand{\N}{\mathbb{N}}

\newcommand{\R}{\mathbb{R}}

\renewcommand\P{\mathbb{P}}

\newcommand{\E}{\mathbb{E}}

\def\R{\mathbb{R}}
\def\C{\mathbb{C}}
\def\N{\mathbb{N}}
\def\E{\mathbb{E}}
\def\P{\mathbb{P}}

\def\var{\mathrm{Var}}

\def\d{\mathrm{d}}
\def\det{\mathrm{Det}}
\def\op{\mathrm{op}}

\def\cF{\mathcal{F}}
\def\cS{\mathcal{S}}

\def\op{{\mathrm{op}}}

\def\iid{{\mathrm{iid}}}
\def\poi{{\mathrm{Poi}}}
\def\dpp{{\mathrm{DPP}}}
\def\d{\mathrm{d}}
\def\span{\mathrm{span}}
\def\bK{\mathbf{K}}
\def\bL{\mathbf{L}}
\def\bI{\mathbf{I}}
\def\Cov{\mathrm{Cov}}
\def\re{\mathrm{Re}}
\def\im{\mathrm{Im}}

\makeatletter
\newcommand{\addresseshere}{%
  \enddoc@text\let\enddoc@text\relax
}
\makeatother

\makeatletter
\renewcommand*{\@cite@ofmt}{\hbox}
\makeatother

\usepackage{xcolor}

\begin{document}


\title{Negative Dependence as a toolbox for machine learning : review and new developments }
\author{
        H.S. Tran{$^{1,\dag}$}
        \and
        V. Petrovic{$^{2}$}    
        \and
	R. Bardenet{$^{2}$}
	\and
	S. Ghosh{$^{1}$}        
}

\date{}
\maketitle

\vspace{20pt}

\begin{center}
\textit{Dedicated to the memory of Prof K.R. Parthasarathy : \\  visionary, guru, and scientist par excellence}
\end{center}

\vspace{50pt}

\begin{abstract}
Negative dependence, where a stochastic system promotes diversity among its variables, is increasingly becoming a key driver in advancing learning capabilities beyond the limits of traditional independence. Recent developments have evidenced support towards negatively dependent systems as a learning paradigm in a broad range of fundamental machine learning challenges—including optimization, sampling, dimensionality reduction and sparse signal recovery, often surpassing the performance  of current state-of-the-art methods based on statistical independence. 
The most popular negatively dependent stochastic model in learning has been that of determinantal point processes (\textit{abbrv.} DPPs), which have their origins in the study of quantum theory and Fermionic systems. However, other models, such as perturbed lattice models, general strongly Rayleigh measures, zeros of random functions have steadily gained salience in various learning applications.

In this article, we undertake a review of this burgeoning field of research, as it has developed over the past two decades or so. Additionally, we present new results on applications of DPPs to the parsimonious representation of neural networks, a problem of significant contemporary interest. In the limited scope of the present article, we mostly focus on aspects of this area to which the authors contributed over the recent years, including applications to Monte Carlo methods, coresets and stochastic gradient descent (SGD), feature selection, stochastic networks, signal processing via spectrogram zeros and connections to quantum computation. However, starting from basics of negative dependence for the uninitiated reader, extensive references are provided to a broad swath of related developments which could not be covered within our limited scope. While existing works and reviews generally focus on specific negatively dependent models (such as DPPs), a notable feature of this article is that it addresses negative dependence as a machine learning methodology as a whole. In this vein, it covers within its span an array of negatively dependent stochastic models and their applications well beyond DPPs, thereby putting forward a very general and rather unique  perspective.

\end{abstract}

\medskip

\maketitle

\vspace{-1cm}

\def\thefootnote{$\dag$}\footnotetext{Corresponding author}

\def\thefootnote{$1$}\footnotetext{Department of Mathematics, National University of Singapore, 10 Lower Kent Ridge Road, 119076, Singapore}

\def\thefootnote{$2$}\footnotetext{Univ. Lille, CNRS, Centrale Lille, UMR 9189 -- CRIStAL, F-59000 Lille, France}

\newpage 

\tableofcontents

\vspace{-1cm}




\newpage

\section{Introduction}
\label{sec: intro}
\input{introduction}

\section{Determinantal point processes}
\label{sec: DPP}
\input{DPP}

\section{Sampling from DPPs}
\label{sec: SamplingDPP}
\input{DS}

\section{DPPs for Monte Carlo}
\label{sec: Monte Carlo}
\input{montecarlo}

\section{DPPs and directionality in data}
\label{sec: GDP}
\input{directionality}

\section{DPPs as a sampling device in machine learning}
\label{sec: coreset}
\input{coreset}

\section{Negative dependence for stochastic networks}
\label{sec: networks}
\input{networks}

\section{Spectrogram analysis and Gaussian Analytic Functions}
\label{sec: spectrogram}
\input{spectrogram}

\section{DPPs for feature selection}
\label{sec: feature selection}

\input{columnsubsetselection}

\section{New results on DPPs for neural network pruning}
\label{sec: NN}
\input{nndpp}

\section{A quantum sampler for DPPs}
\label{sec: quantum}
\input{quantum}

\section{Acknowledgements}
SG was supported in part by the MOE grants R-146-000-250-133, R146-000-312-114, A-8002014-00-00 and MOE-T2EP20121-0013. 
H.S.T. was supported by the National University of
Singapore and MOE of Singapore through the grant A-8003576-00-00.
VP and RB were supported by grants ERC-2019-STG-851866 and ANR20-CHIA-0002.
The authors would like to thank Kin Aun Tan and Clarence Chew from the National University of Singapore for their comments and feedback on the paper.

\newpage

\bibliographystyle{alpha}
\bibliography{references}

\end{document}

%% file: introduction.tex
\subsection{Background}
The traditional framework of statistical independence has significantly advanced the fields of machine learning and statistics. However, its applicability and effectiveness may be reaching their limits in certain contexts, both as a modeling approach and as a foundation for algorithm design. This presents a strong motivation to explore alternative learning paradigms. Notable among these are models inspired by statistical physics that exhibit strong correlations, an approach that has attracted significant attention in recent years. The goal is to take advantage of the global dependence structures and long-range correlations that characterize these processes to enhance efficiency and reduce complexity in achieving improved learning objectives.

Negative dependence, where a stochastic system encourages diversity among its variables, is emerging as a key driver for advancing learning capabilities beyond the constraints of classical independence. In a range of core learning challenges — including but not limited to optimization, sampling, dimensionality reduction, and sparse signal recovery — recent research has shown that negatively dependent systems outperform state-of-the-art methods based on statistical independence, positioning them as a powerful new approach for future developments in this field.

\subsection{Motivation and heuristics}
A typical instantiation of this approach may be seen in the setting of quadratures, a topic that will be covered in a more technical and in-depth discussion later in this article. In this introductory segment, we will undertake a heuristic discussion for illustrative purposes. 

\begin{figure}[h]
    \centering
    \includegraphics[width=5cm, height=5cm]{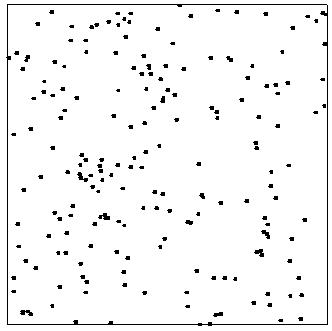}
    \includegraphics[width=5cm, height=5cm]{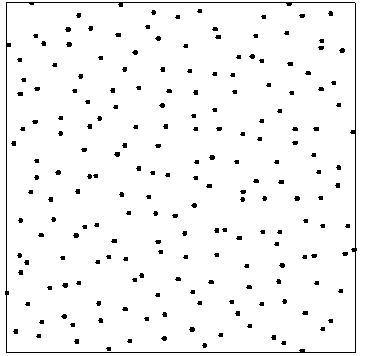}
     \includegraphics[width=5cm, height=5cm]{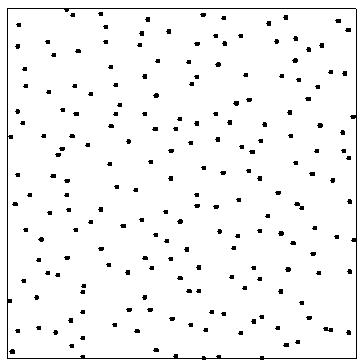}
     \caption{(a) Left: Poisson  \qquad (b) Center: DPP \qquad (c) Right: GAF}
    \label{fig:poisson-dpp-gaf}
\end{figure}

The panel of Fig. \ref{fig:poisson-dpp-gaf} demonstrates three competing point processes on a 2D square (with the same average density of points per unit area): (a) uniformly chosen random points, (b) a determinantal point process (abbrv DPP), and (c) zeros of a Gaussian random series. While the first one is a representation of the basic i.i.d. notion of randomness that is classically popular in probability, statistics, and applied mathematics, the latter two are strongly correlated models of randomness, exhibiting \textit{negative association} to varying degrees. The negative association entails that the random point set model incentivizes points to \textit{repel} each other (esp. at short ranges) -- in more technical terms, the stochastic model puts more weight on point configurations with fewer pairs of points that are close to each other. For the DPP models, this feature holds true at all distance scales, whereas for the Gaussian zero processes this is known to hold true at short distances (the most important regime for repulsion in practical scenarios) (c.f. Eq. \ref{eq:DPP_repel}, Eq. \ref{eq:Poisson_repel} in Section \ref{sec: DPP} and the discussion about GAFs in the end of subsection 7.1, Section \ref{sec: spectrogram} for a more analysis of this phenomenon).

Based on these considerations, we present below a heuristic argument to demonstrate how negatively associated random point sets can contribute to improved learning outcomes and approximation guarantees. To this end, let us consider a simple scenario where we use the sampled point set $\Xi$ for \textit{quadrature} -- namely, for a reasonably regular real-valued function $f$ defined on the square $D$ (c.f. Fig. \ref{fig:poisson-dpp-gaf}), we desire to use a (weighted) average of $f$ over $\Xi$ in order to approximate the integral $\int_D f(x) \d x$. A traditional baseline in randomized quadrature would be to use uniformly sampled independent points (c.f. Fig. \ref{fig:poisson-dpp-gaf} panel (a)), which we compare against negatively dependent samples (such as in Fig. \ref{fig:poisson-dpp-gaf} panel (b) and (c)). 

Observe that the independent point set in panel Fig. \ref{fig:poisson-dpp-gaf} (a) exhibits clusters of points that are relatively closer to each other, interspersed with vacant spaces. In comparison, the negatively dependent point sets in Fig. \ref{fig:poisson-dpp-gaf} (b) and (c) exhibit a much more homogeneous spatial outlay of points. In fact, they almost resemble a randomly perturbed or stochastic grid, which would incidentally be another viable model of negatively dependent random point set for applications (see Sec. \ref{sec: networks}). Such behaviour can, in fact, be seen from negative dependence at short length scales via a simple heuristic argument. To wit, viewed as a physical particle system, the mutual repulsion tends to push the particles away from each other. But the confining potential (in physical terms), equivalently the background measure (in stochastic terms), prevents the particles from escaping to infinity. The tension between these two opposing effects compels the particles to equilibriate around a spatially homogeneous point configuration, as observable in Fig \ref{fig:poisson-dpp-gaf} panels (b) and (c).

Coming back to the question of randomized quadrature, notice that any estimate of $\int_D f(x) \d x$ based on independent points will be overly dependent on the functional behaviour in the region of point clusters, and relatively uninformative about the vacant regions with relatively few or no points. This introduces inherent inefficiencies in the quadrature procedure. On the other hand, owing to their homogeneous spatial distribution, a negatively dependent point set captures the functional behaviour across spatial regions in a more comprehensive manner, thereby leading to more robust approximation properties. 

On a related note, the fact that the expected density of points per unit area is constant necessitates that in different samples of the independent point sets, the location of the point clusters and vacant spaces (i.e., regions with relatively higher and lower density of points ) keep changing, which implies that quadrature estimates based on such point sets would exhibit large fluctuations from one sample to another. On the other hand, quadrature estimates based on negatively dependent point sets tend to be much more stable from one quadrature sample to another, thereby leading to much tighter approximation guarantees.

The above considerations are further reinforced by the fact that it is possible to develop a theory of negatively dependent random point sets on very general background spaces, widely extending the ambit of the heuristics presented above. This includes, in particular, settings where the notion of a natural grid is not available (e.g., manifolds with non-trivial curvature such as spheres, spaces without geometry such as combinatorial objects like graphs and discrete spaces which are a staple in many machine learning problems).

\subsection{Related literature}
Random point sets with negative dependence have been in vogue as a scientific discipline for several decades. The principal motivation for this came from statistical and quantum physics. An early connection arose in the context of Fermionic particle systems, wherein determinantal structures were observed in ground state densities of Fermionic systems via Slater determinants \cite{Slater}. Subsequently, it was developed as a mathematical discipline by Macchi and others \cite{Mac72, Mac75, SoshDPP}, and new connections arose with a wide array of areas in mathematics and statistical physics, including random matrices and Coulomb gases  \cite{johanssondpp, borodin}, integrable systems  \cite{Deift_1,Deift_2}, combinatorics \cite{borodin2015_integrable,borodin2016_integrable}, among others. 

Negative dependence as a tool for machine learning and statistics is of a more recent historical provenance. Earlier work in this direction is well explored  in the excellent survey by Kulesza and Taskar  \cite{kulesza_determinantal_2012}, which appeared only around a decade ago, where we refer the interested reader for an account of formative developments. This was also roughly contemporaneous with the emergence of works in the statistical literature that looked at DPPs and allied negatively dependent processes as a statistical modelling paradigm for strongly correlated data \cite{LaMoRu14, lavancier-1,lavancier-2}. These developments motivated investigations into statistical estimation and inference on DPPs, e.g., learning the parameters of DPPs \cite{Fox-Affandi-1, Rigollet-2},  efficient sampling schemes for continuous DPPs \cite{Fox-Affandi-2},  rates of estimation for DPPs \cite{Rigollet-1}, and DPPs for minibatch sampling in stochastic gradient descent \cite{OPE-NIPS}.  

This also ties in with a burgeoning interest, principally in the theoretical computer science and adjacent communities, in carrying out investigations into general models of negatively associated random sets (in the discrete setting), known as \emph{strongly Rayleigh measures}. 
Within the limited boundaries of the present article, we will not have the occasion to delve into this beautiful theory, instead referring the reader to the seminal works of Borcea, Branden and Liggett \cite{Borcea-Branden, Borcea-Branden-2, Borcea-Branden-3, Borcea-Branden-Liggett, BRANDEN, Branden-2}, for the development of the general theory as well as its connections to other parts of discrete mathematics, such as matroid theory, log concavity of sequences and related topics. In the setting of applications to learning theory, notable contributions include counting bases of matroids \cite{anari_logconcave2}, sampling algorithms \cite{anari_mcmc_rayleigh_dpp, Jegelka}, traveling salesman problem \cite{travelling-salesman}, learning with DPPs \cite{gartrell2020scalable,Sra,mariet_diversity_2017}. 
On a related note, we must mention the resolution of the Kadison-Singer conjecture and construction of a very general class of Ramanujan graphs by Marcus, Spielman and Srivastava over the last decade that strongly leverages zeros of characteristic polynomials and their interlacing structures \cite{interlacing-1, Marcus-Spielman-Srivastava, interlacing-3, interlacing-4}.

In this article, we will review recent developments focusing on the use of negative dependence as a tool for improving the state of the art in fundamental learning problems. We will make the subtle distinction of this approach from the use of DPPs as a statistical modelling paradigm (cf. references above), a topic of its own interest that we will have the occasion to address only in passing within the limited span of this article. In the main, we will largely focus in detail on problems in this direction where the authors have had the opportunity to make a contribution, with elaborate discussions on connections to other related topics. These include DPPs for Monte Carlo methods, DPPs and directionality in data, DPPs for coresets in machine learning, negative dependence for stochastic networks, zeroes of Gaussian analytic functions for spectrogram analysis, DPPs for feature selection, and a quantum sampler for DPPs. We will also present new theoretical results on applications of DPPs in the context of neural network pruning, a problem that has attracted considerable interest in recent years.

%% file: DPP.tex
\subsection{Basic notions}

\emph{Determinantal point processes} (DPPs) are random configurations of locally finite point sets over some background space, whose correlation functions are given by determinants of certain kernels. They have a long history in both mathematics and physics. DPPs were first introduced by Macchi \cite{Mac72}, motivated by fermions in quantum mechanics. Macchi discovered that DPPs describe the distribution of a fermionic system at thermal equilibrium; their repulsive behavior captures the Pauli exclusion principle, which states that two fermions cannot occupy the same quantum state. Surprisingly, DPPs also arise naturally in many mathematical settings, such as the eigenvalues of random matrices, random spanning trees, and zeros of random analytic functions, among others. As such, DPPs have intertwined with various fields in mathematics, including random matrix theory \cite{borodin,johanssondpp}, integrable systems \cite{Deift_1,Deift_2}, combinatorics \cite{borodin2015_integrable,borodin2016_integrable}, and complex geometry \cite{Berman0,Berman1,Berman2,Berman3}, to name a few.

Formally, let $\mathcal X$ be a Polish space equipped with a positive Radon measure $\mu$ (e.g., $\mathbb R^d$ with the Lebesgue measure or a discrete set with the counting measure).
A \emph{point process} $\mathcal S$ on $\mathcal{X}$ is a random $\mathbb N$-valued Radon measure on $\mathcal X$. If $\mathcal S$ assigns mass at most $1$ for each point almost surely, we say that $\mathcal S$ is a \emph{simple point process}.

\begin{definition}[DPPs] \label{def:GeneralDPPs}
    A point process $\mathcal S$ is said to be determinantal if there exists a measurable function $K : \mathcal{X} \times \mathcal{X} \rightarrow \mathbb{C}$ such that, for all $n \geq 1$, if $f : \mathcal{X}^n \rightarrow \mathbb{R}$ is a bounded measurable function:
    \[ \mathbb{E}\Big [\sum_{\neq} f(x_{i_1}, \cdots, x_{i_n}) \Big ] = \int _{\mathcal{X}^n}f(x_1, \cdots x_n) \det((K(x_i,x_j))_{1 \leq i,j \leq n}) \d\mu^{\otimes n}(x_1, \cdots, x_n), \]
    where the sum in the LHS ranges over all pairwise distinct $n$-tuplets of $\mathcal{S}$.
    We then call $\mathcal S$ a DPP on $\mathcal X$ with kernel $K$ and reference measure $\mu$.
\end{definition}

\begin{remark}
    It can be shown that DPPs are simple, thus one can visualize DPPs as random configurations of locally finite point sets on $\mathcal X$. 
\end{remark}

In the theory of point processes, an important object to study is the so-called \emph{linear statistics}.
\begin{definition}[Linear statistics]
    For a point process $\mathcal S$ on $\mathcal X$ and a measurable function $f: \mathcal X \rightarrow \mathbb C$, the linear statistic $\Lambda_{\mathcal S}(f)$ is defined as
\[\Lambda_{\mathcal S}(f) := \sum_{x\in \mathcal S} f(x).\]
\end{definition}
Under mild conditions, the joint distribution of $\Lambda_\cS(f)$ for a sufficiently rich class of test functions $f$ (e.g., the class of all continuous, compactly supported functions) can determine the distribution of the process $\cS$; thus, studying linear statistics is of fundamental interest. For DPPs, due to the determinantal structure, there is an explicit formula for the Laplace transform of linear statistics. As a consequence, one can deduce explicit formulas for the cumulants (hence, moments) of linear statistics of DPPs (e.g., see \cite{lambertdpps}).
\begin{proposition}
    Let $\cS$ be a DPP on $\mathcal X$ with kernel $K$ and reference measure $\mu$. Let $f:\mathcal X \rightarrow \mathbb C$ be a bounded, compactly supported measurable function. Then
    \begin{eqnarray*}
        \mathbb E[e^{t\Lambda_\cS(f)}] &=& \det[I + M_f \mathcal K],
    \end{eqnarray*}
    where $\det$ denotes the Fredholm determinant, $M_f$ is the multiplication operator by the function $e^{tf(x)}-1$, acting on $L^2(\mathcal X,\mu)$, and $\mathcal K$ denotes the integral operator acting on $L^2(\mathcal X,\mu)$ induced by the kernel $K$.
\end{proposition}

We warn the readers that not every kernel $K$ defines a DPP. In fact, we have the following theorem by Macchi \cite{Mac75} and Soshnikov \cite{SoshDPP}.

\begin{theorem}[Macchi-Soshnikov]
    For a kernel $K: \mathcal X \times \mathcal X \rightarrow\mathbb C$, we denote by $\mathcal K: L^2(\mathcal X,\mu) \rightarrow L^2(\mathcal X, \mu)$ the integral operator induced by $K$. Let $K$ be a kernel such that $\mathcal K$ is self-adjoint and locally trace class. Then $K$ defines a DPP (w.r.t $\mu$) if and only if all the eigenvalues of $\mathcal K$ are in $[0,1]$.
\end{theorem}

An important subclass of DPPs, called \emph{projection DPPs}, consists of those for which the associated integral operator $\mathcal K: L^2(\mathcal X,\mu) \rightarrow L^2(\mathcal X,\mu)$ is self-adjoint, bounded, and has all non-zero eigenvalues equal to 1 (equivalently, $\mathcal K$ is a projection operator onto a subspace of $L^2(\mathcal X,\mu)$). It turns out that all DPPs with Hermitian kernels are mixtures of projection DPPs (see \cite{HKPV}).

Similar to Gaussian processes, all statistical properties of a DPP are encoded in the
kernel $K$ and background measure $\mu$ (or equivalently, the integral operator $\mathcal K$). For example, for any compact subset $D\subset \mathcal X$, the number of points of the DPP defined by $\mathcal K$ lying inside $D$ is equal in distribution to a sum of independent Bernoulli random variables with parameters the eigenvalues of $\mathcal K_D$, the integral operator on $L^2(\mathcal X,\mu)$ defined by the kernel $\mathbf{1}_D(x)K(x,y)\mathbf{1}_D(y)$. In particular, for projection DPPs, the total number of points is almost surely a constant, which is equal to $\rank(\mathcal K)$.

DPPs are known to exhibit \emph{repulsiveness}, i.e., particles tend to repel each other. This entails, for example, that the probability of having multiple points inside the same small neighbourhood would be much smaller for a DPP compared to that for a Poisson process with the same first intensity. To see this, in what follows we present a simple computation.

Let $\mathcal S$ be a DPP on $\mathbb R^d$ with Hermitian, translation-invariant kernel $K$ (i.e., $K(x,y) = \overline{K(y,x)}$ and $K(x,y)=\Phi(x-y)$ for some $\Phi$), and background measure the Lebesgue measure. Let $B_\varepsilon$ be a ball of small radius $\varepsilon>0$ (one can assume the center is $0$ due to the translation invariance) and denote by $\mathcal S(B_\varepsilon)$ the number of points in $\mathcal S$ inside $B_\varepsilon$. Then we have
\begin{eqnarray*}
    \mathbb P(\mathcal S(B_\varepsilon) \ge 2) &\le& \mathbb E[\mathcal S(B_\varepsilon)(S(B_\varepsilon)-1)] \\
    &=&\iint_{B_\varepsilon^2} \Big( K(x,x)K(y,y) - K(x,y)K(y,x) \Big ) \d x \d y\\
    &=&\iint_{B_\varepsilon^2} \Big( |\Phi(0)|^2 - |\Phi(x-y)|^2 \Big ) \d x \d y \\
    &=& \varepsilon^{2d} \iint_{B(0,1)^2} \Big( |\Phi(0)|^2 - |\Phi(\varepsilon(u-v))|^2 \Big ) \d u \d v.
\end{eqnarray*}
Thus, for sufficiently regular $\Phi$, 
we would have
\begin{equation} \label{eq:DPP_repel}
    \mathbb P(\mathcal S(B_\varepsilon) \ge 2) = \mathcal{O}(\varepsilon^{2d+1}) \quad \text{as } \varepsilon\rightarrow 0.
\end{equation}

In contrast, for Poisson point processes, the order of the probability above is $\varepsilon^{2d}$. To elaborate, let $\mathcal S_\poi$ be a Poisson point process on $\mathbb R^d$ with the same first intensity $K(x,x)\d x = \Phi(0) \d x$ as $\mathcal S$. Then, by definition, $\mathcal S_\poi(B_\varepsilon)$ is a random variable following the Poisson distribution with parameter $\lambda :=\Phi(0) \cdot\mathrm{Vol}(B_\varepsilon)$. Thus,
\begin{equation} \label{eq:Poisson_repel}
    \mathbb P(\mathcal S_\poi(B_\varepsilon) \ge 2) = e^{-\lambda}(e^\lambda - 1 - \lambda) \asymp \varepsilon^{2d} \quad \text{as } \varepsilon\rightarrow 0,
\end{equation}
where we used the fact that $\lambda = \Phi(0) \cdot \mathrm{Vol}(B(0,1)) \cdot \varepsilon^d$.

\medskip  

For more detailed discussion and further properties of DPPs, we refer the readers to \cite{SoshDPP,HKPV,LyonsIHES,lyons2014survey} for excellent overviews. 

\subsection{DPPs in machine learning}
In recent years, DPPs have attracted the attention of the machine learning community. They are generic statistical models for repulsion in
spatial statistics \cite{LaMoRu14, biscio2017contrast} and machine learning \cite{kulesza_determinantal_2012, belhadji_determinantal_2020, gartrell2019learning, gartrell2020scalable, determinantal-averaging, tremblay2019determinantal, GDP}. Sampling and inference with DPPs are tractable, and there are efficient algorithms for such tasks \cite{kulesza_determinantal_2012, GDP, gartrell_scalablesampling, han2022scalable}.

In the setting of machine learning, the background set (sometimes called the \emph{data set}) is usually finite, say, $\mathcal X = \{1,..,N\}$. It is then common to use the counting measure as the background measure on $\mathcal X$. In this setting, a more intuitive way to define DPPs is as follows. 

\begin{definition}[Discrete DPPs] \label{def:DPPdef}
    A random subset $\mathcal S$ of $\mathcal X$ is called a DPP if there exists an $N\times N$ matrix $\bK$ such that
    \[ \mathbb P(A\subset \mathcal S) = \det [\bK_{A}], \quad \forall A\subset \mathcal X,\]
    where $\bK_A$ denotes the submatrix of $\bK$ with rows and columns indexed by the subset $A$.
\end{definition}

    The Macchi-Soshnikov theorem is then translated to that if $\bK$ is a Hermitian matrix, then $\bK$ defines a DPP if and only if $\mathbf{0} \preceq \bK \preceq \mathbf I $. 
    In this setting, projection DPPs are precisely those whose kernels are projection matrices.
    
    If all eigenvalues of $\bK$ are strictly less than $1$, the DPP defined by $\bK$ is also an $L$-ensemble, namely
    \[ \mathbb P(\mathcal S= A) \propto \det [\bL_A], \quad \forall A\subset \mathcal X\]
    where $\bL:= \bK(\bI - \bK)^{-1}$ and $\bL_A$ denotes the submatrix of $\bL$ whose rows and columns indexed by the subset $A$ (see \cite{kulesza_determinantal_2012}).

In general, the cardinality $|\cS|$ of a DPP $\cS$ is an $\mathbb N$-valued random variable. In applications, sometimes it is interesting to sample a point process similarly to a DPP but with fixed size $m$ (called an $m$-DPP in  \cite{kulesza_determinantal_2012}). This can be done by a conditioning argument as follows. Given a DPP $\mathcal S$ on $\mathcal X$, one can obtain an $m$-DPP by considering the conditioned process $(\cS| \{|\cS| = m\} )$, provided that $\mathbb P (|\cS|= m ) >0$. However, $m$-DPPs do not have the determinantal structure in general (i.e., they are not DPPs); the only exception is when the original DPP $\mathcal S$ is a projection DPP.

%% file: DS.tex

DPPs are particularly useful tools for machine learning since they can be sampled in polynomial time. The classic algorithm to sample a DPP comes from \cite{HKPV} which leverages a neat geometric understanding of a DPP as a probability measure. Recent algorithmic advances significantly improve this classical algorithm to achieve computational efficiency. 

\subsection{The classic Spectral Sampler for DPPs} \label{sec:HKPV}

Let $\mathcal S$ be a projection DPP of rank $m$, i.e., its kernel $K$ defines a projection operator $\mathcal K_H$ from $L^2(\mathcal X,\mu)$ to an $m$-dimensional subspace $H \subset L^2(\mathcal X,\mu)$. We denote by $\|\cdot\|$ the norm in the Hilbert space $L^2(\mathcal X,\mu)$. Then for each $f\in L^2(\mathcal X,\mu)$, $\mathcal K_H f$ is the unique element in $H$ closest to $f$.

In the discrete case, the delta mass at $x$ for $x\in \mathcal X$ is an element in $L^2(\mathcal X,\mu)$, namely
\[  \delta_x (y) = \begin{cases} 
    1/\mu(\{x\}) & \text{ if } y = x \\
    0 & \text{ otherwise.}
\end{cases}\]
Thus, $\mathcal K_H$ can act on $\delta_x$ and by direct calculations, one has $\mathcal K_H \delta_x (\cdot ) = K(\cdot ,x)$. In the general case,
we define $\mathcal K_H \delta_x (\cdot) := K(\cdot, x)$. Intuitively, $\mathcal K_H \delta_x$ represents the element in $H$ which is closest to $\delta_x$.

Since $K$ is a projection kernel, one can easily check that
\[ \| \mathcal K_H\delta_x\|^2 = \int |K(y,x)|^2 \d \mu(y) = \int K(x,y) K(y,x) \d \mu(y) = K(x,x).\]
The intensity measure of $\mathcal S$ is then given by
\[ \d \mu_H(x) = K(x,x) \d \mu(x) = \|\mathcal K_H\delta_x\|^2 \d \mu(x).\]
We note that $\mu_H(\mathcal X) = m$, so $\frac{1}{m}\mu_H$ is a probability measure on $\mathcal X$. Starting with $H_m := H$, the algorithm to sample the DPP $\mathcal S$ is as follows.

\begin{algorithm} [H] \label{alg:HKPV}
\caption{Sampling a projection DPP}
\begin{algorithmic}[1]
    \State If $m=0$, stop.
    \State Pick a random point $X_m$ in $\mathcal X$ from the probability measure $\frac{1}{m}\mu_{H_m}$.
    \State Let $H_{m-1} \subset H_m$ be the orthogonal complement of the function $\mathcal K_{H_m}\delta_{X_m}$ in $H_m$. In the discrete case, $H_{m-1} = \{f\in H_m : f(X_m) =0 \}.$ Note that $\dim H_{m-1} = m-1$ almost surely. 
    \State Decrease $m$ by $1$ and iterate.
\end{algorithmic}
\end{algorithm}

\begin{proposition}
 Let $(X_1,\ldots,X_m)$ be the points constructed by Algorithm 1. Then $ \mathcal S \overset{d}{=}  \{X_1,\ldots,X_m\} $.
\end{proposition}

\begin{proof}
    We compute the density of $(X_1,\ldots,X_m)$ at $(x_1,\ldots,x_m) \in \mathcal X^m$ with respect to $\mu^{\otimes m}$ (where $x_i \in \mathcal X$ are  distinct). Let $\Psi_j:= \mathcal K_H \delta_{x_j}$, then one can easily check that $\mathcal K_{H_j}\delta_{x_j} = \mathcal K_{H_j}\Psi_j$. Thus, the density of $(X_1,\ldots,X_m)$ at $(x_1,\ldots,x_m)$ equals
    \[ p(x_1,\ldots,x_m) = \prod_{j=1}^m \frac{\|\mathcal K_{H_j}\Psi_j \|^2}{j} \cdot\]
    We further note that $H_j = H \cap \langle \Psi_{j+1},\ldots, \Psi_m \rangle^\perp$. Therefore, the quantity $\prod_j \|\mathcal K_{H_j} \Psi_j\|^2$ equals the determinant of the Gram matrix whose $(i,j)$ entry is given by the inner product of $\Psi_i$ and $\Psi_j$. By direct computation,
    $ \int \Psi_i \bar\Psi_j \d \mu = K(x_i,x_j)$.
    Thus
    \[ p(x_1,\ldots,x_m) = \frac{1}{m!} \det [K(x_i,x_j)]_{m\times m}.\]
    Hence, the point process $\{X_1,\ldots,X_m\}$ is a DPP on $\mathcal X$ defined by the kernel $K$ and background measure $\mu$, as desired.
\end{proof}

In the discrete setting, when the DPP $\mathcal S$ has the structure of an $L$-ensemble with kernel $\bL = \bK(\bI-\bK)^{-1}$, we can sample $\mathcal S$ via a spectral argument as follows. 
Let $\bL=\sum_{i =1}^N \lambda_i \mathbf{v}_i \mathbf{v}_i^\top$ be the spectral decomposition of $\bL$. The spectral decomposition of $\bK$ is then given by: $\bK=\sum_{i =1}^N (\lambda_i/(1+\lambda_i)) \mathbf{v}_i \mathbf{v}_i^\top$.

\begin{algorithm} [H]
\caption{Spectral sampling}
\begin{algorithmic}[1]
    \State \textbf{Input:} $(\mathbf v_i, \lambda_i)_{i \in [N]}$ an eigendecomposition of $\bL$
    \State $V \gets \emptyset$, $\mathcal S \gets \emptyset$
    \For {$1 \leq j \leq N$}
        \State Add $\mathbf v_j$ to $V$ with probability $\lambda_j/(1+\lambda_j)$
    \EndFor
    \While {$\lvert V \rvert > 0$}
        \State Sample $i \in [N]$ with probability: $\frac{1}{\lvert V \rvert} \sum \limits_{\mathbf v \in V} (\mathbf v^\top \mathbf e_i)^2 $ and add it to $\mathcal S$
        \State $V \gets V_\perp$ an orthonormal basis of the subspace of $V$ orthogonal to $\mathbf e_i$
    \EndWhile
    \State \textbf{Return} $\mathcal S$
\end{algorithmic}
\end{algorithm}
This algorithm is an avatar of the complete algorithm of \cite{HKPV}, and the proof of validity in the discrete form we use here can be found in \cite{kulesza_determinantal_2012}.
When $N$ becomes large, the true bottleneck here is accessing the eigendecomposition of $\mathbf{L}$, which required $\mathcal{O}(N^3)$ operations. For kernels with low rank $r \ll N$, the computational cost of the eigendecomposition reduces to $\mathcal{O}(r^2N)$ by using dual kernel. 

\subsection{Tree-based algorithms for computational efficiency}
Recent algorithmic advances have introduced tree-based methods to perform expedited sampling of DPPs, focusing on the setting where repeated DPP samples are required via the same kernel. This situation is very common in machine learning applications, such as Stochastic Gradient Descent (abbrv. SGD), recommender systems, coresets, etc. 

When repeated DPP samples are required, economy can be achieved by reusing the eigendecomposition of the kernel, to be computed in the preprocessing step. For the second phase, given the eigendecomposition, the complexity of getting each additional sample is linear in $N$ in Algorithm 2. This dependence can be significantly reduced to $\log N$ by using a tree-based algorithm for fast repeated sampling from DPPs, developed recently in \cite{kulesza-gillenwater}. 

{
 To elaborate, let $\mathcal S$ be an $L$-ensemble with kernel $\mathbf{L}\in\mathbb{R}^{N\times N}$ of rank $r$. We write $\mathbf{L}=\mathbf{B}^\top \mathbf{B}$, where $\mathbf{B}\in\mathbb{R}^{r\times N}$ and define $\mathbf{C}:=\mathbf{B}\mathbf{B}^\top \in \mathbb{R}^{r\times r}$, called the \textit{dual kernel}. In practice, when the user provides $\mathbf{B}$, the dual kernel $\mathbf{C}$ is preferred in sampling low-rank DPPs, as it reduces the eigendecomposition cost to $\mathcal{O}(r^2N)$. We consider the eigendecomposition of $\mathbf{C}=\sum_{i=1}^r \lambda_i \mathbf{w}_i\mathbf{w}_i^\top$. The first phase of the tree-based algorithm in \cite{kulesza-gillenwater} is the same as the first phase in Algorithm 2, where we use the eigendecomposition to sample a subset of indices $E\subset\{1,2,\ldots,r\}$ by independently picking $i\in E$ with probability $\lambda_i/(\lambda_i+1)$. The advancement comes in the second phase, where instead of directly sampling from a multinomial distribution defined by a projection kernel (obtained in the first phase) and repeatedly updating the kernel as in Algorithm 2 (which is costly), we will repeatedly sample from a binary tree of depth $\log N$. The tree construction, appropriate summary statistics stored in the tree will be done in the pre-processing phase, which is a one-time cost.  }

{
The construction of the tree is as follows. The root of the tree is the full set of items $[N]$, and each child node is half of its parent's items. We first precompute two $r\times N$ matrices $\mathbf{G}$ and $\mathbf{H}$ 
\[ \mathbf{G}_{ij} := \mathbf{b}_j^\top \mathbf{w}_i \quad,\quad \mathbf{H}_{ij} := \gamma_i \mathbf{G}_{ij}\]
where $\mathbf{b}_j$'s are the columns of $\mathbf{B}\in\mathbb{R}^{r\times N}$ and $\gamma_i:=1/\lambda_i$ for $1\le i \le r$.
The $j^{th}$ columns of $\mathbf{G}$ and $\mathbf{H}$ will be denoted by $\mathbf{g}_j$ and $\mathbf{h}_j$ respectively.
At each tree node $S \subset [N]$, we store a vector $\mathbf{z}^{(S)}\in\mathbb R^r$ and a matrix $A^{(S)}\in \mathbb{R}^{r\times r}$, given by
\[\mathbf{z}^{(S)}_i := \gamma_i \sum_{j\in S} \mathbf{G}_{ij}^2 \quad,\quad \mathbf{A}^{(S)}_{i_1 i_2} := \sum_{j\in S} \mathbf{H}_{i_1 j} \mathbf{H}_{i_2 j}.\]
We remark that a recursive argument can be applied here to efficiently compute these quantities. Indeed, if $S=S_l \cup S_r$, where $S_l$ and $S_r$ are child nodes of $S$, then 
\[\mathbf{z}^{(S)} = \mathbf{z}^{(S_l)} + \mathbf{z}^{(S_r)}\quad,\quad \mathbf{A}^{(S)} = \mathbf{A}^{(S_l)} + \mathbf{A}^{(S_r)}.\]
The algorithm for tree construction is described below. Here note that $\mathrm{Split}$ is a function that partitions a set into two subsets of approximately equal size, and $\circ$ denotes the entrywise product.}
\begin{algorithm} [H] \label{alg:tree-const}
\caption{Tree construction}
\begin{algorithmic}[1]
    \State \textbf{procedure} $\mathrm{ConstructTree}(S,\gamma,\mathbf{G},\mathbf{H})$
    \State \quad \textbf{if} $S=\{j\}$ \textbf{then}
    \State \quad \quad $\mathbf{z} \leftarrow \gamma \circ \mathbf{g}_j^2$
    \State \quad \quad $\mathbf{A} \leftarrow \mathbf{h}_j \mathbf{h}_j^\top$
    \State \quad \quad \textbf{return} $\mathrm{Tree}(\mathbf{z},\mathbf{A},j)$
    \State \quad $S_l, S_r \leftarrow \mathrm{Split}(S)$ 
    \State \quad $\mathsf{T}_l \leftarrow \mathrm{ConstructTree}(S_l,\gamma,\mathbf{G},\mathbf{H})$
    \State \quad $\mathsf{T}_r \leftarrow \mathrm{ConstructTree}(S_r,\gamma,\mathbf{G},\mathbf{H})$
    \State \quad $\mathbf{z}\leftarrow \mathsf{T}_l.\mathbf{z} + \mathsf{T}_r.\mathbf{z} $
    \State \quad $\mathbf{A}\leftarrow \mathsf{T}_l.\mathbf{A} + \mathsf{T}_r.\mathbf{A} $
    \State \textbf{return} $\mathrm{Tree}(\mathbf{z}, \mathbf{A}, \mathsf{T}_l,\mathsf{T}_r)$
\end{algorithmic}
\end{algorithm}

{
Given the binary tree constructed via Algorithm 3, we sample from the DPP as follows. Let 
\begin{equation}\label{eq:dpp-def-by-E}
    \mathbf{K}:= \sum_{i\in E} \lambda_i^{-1} (\mathbf{B}^\top \mathbf{w}_i)(\mathbf{B}^\top \mathbf{w}_i)^\top
\end{equation}
be the projection kernel obtained in the first phase. Given a subset $Y\subset [N]$ of selected items, 
we remark that the next item can be directly sampled from the ground set $[N]$ using the conditional kernel
\[\mathbf{K}^Y := \mathbf{K}_{\bar Y} - \mathbf{K}_{\bar Y Y} (\mathbf{K}_Y)^{-1} \mathbf{K}_{Y\bar Y}, \]
where $\bar Y:= [N]\setminus Y$, and $\mathbf{K}_{AB}$ is the submatrix of $\mathbf{K}$ whose rows are indexed by $A\subset [N]$ and columns are indexed by $B\subset [N]$, and $\mathbf{K}_A := \mathbf{K}_{AA}$. However, $\mathbf{K}^Y$ could be costly to compute. The key idea is that, instead of computing the conditional kernel, we will find the next item by traversing the tree from the root to one of its leaves. 
Assuming that we have reached the node $S$ in the tree, we will go to its left child node $S_l$ with probability 
$\mathrm{Pr}(S_l | Y):= (\sum_{j \in S_l} \mathbf{K}^Y_{jj})/(\sum_{j \in S} \mathbf{K}^Y_{jj})$,
and go to its right child node $S_r$ with probability $\mathrm{Pr}(S_r|Y) := 1 - \mathrm{Pr}(S_l|Y)$.
We repeatedly doing this until we reach a leaf of the tree, which gives the next item. 
The following proposition shows that we do not have to fully compute the conditional kernel $\mathbf{K}^Y$ to determine $\mathrm{Pr}(S_l|Y)$ and $\mathrm{Pr}(S_r|Y)$. In fact, they can be effectively computed from the information which is already available in the pre-processing phase.
\begin{proposition}[Proposition 1 in \cite{kulesza-gillenwater}] \label{prop:tree-based}
    Let $\mathbf{K}$ be the projection kernel defined by a subset of indices $E\subset \{1,\ldots,r\}$ as in Equation \eqref{eq:dpp-def-by-E}.
    Let $Y$ be a (potentially empty) subset of elements that have already been selected. Then
    \[ \sum_{j\in S} \mathbf{K}^Y_{jj} = \mathbf{1}^\top \mathbf{z}^{(S)}_E - \mathbf{1}^\top \Big [(\mathbf{K}_Y)^{-1} \circ \Big (\mathbf{G}_{EY}^\top \mathbf{A}^{(S)}_E \mathbf{G}_{EY}\Big) \Big ] \mathbf{1},\]
    where $\circ$ denotes entrywise product, and $\mathbf{1}$ is the all-ones vector.
\end{proposition}
  Since the depth of the tree is $\log N$, this tree-based sampling algorithm is in $\mathcal{O}(\log N)$ time, which is a significant improvement over the linea dependence of Algorithm 2. We summarize the argument above in the algorithm below.
}

\begin{algorithm} [H] \label{alg:tree-samp}
\caption{Tree-based sampling}
\begin{algorithmic}[1]
    \State \textbf{procedure} $\mathrm{Sample}(\mathsf{T},\lambda,\mathbf{G},\mathbf{H})$
    \State \quad $E\leftarrow \emptyset, Y \leftarrow \emptyset$
    \State \quad \textbf{for } $i=1,\ldots,r$ \textbf{ do}
    \State \quad \quad $E\leftarrow E\cup\{i\} \text{ w.p. } \lambda_i/(\lambda_i+1)$
    \State \quad $\mathbf{Q}\leftarrow 0\times 0 \text{ matrix}$
    \State \quad \textbf{for } $j=1,\ldots, |E|$ \textbf{ do}
    \State \quad \quad $y \leftarrow \mathrm{SampleItem}(\mathsf{T}, E, Y , \mathbf{G}, \mathbf{Q})$
    \State \quad \quad $Y \leftarrow Y \cup y$
    \State \quad \quad $\mathbf{Q}\leftarrow \mathrm{ExtendInverse}(\mathbf{Q},\mathbf{G}^\top_{Ey}\mathbf{H}_{EY})$
    \State \quad \textbf{return} $Y$
    \State \textbf{procedure} $\mathrm{SampleItem}(\mathsf{T},E,Y,\mathbf{G},\mathbf{Q})$
    \State \quad \textbf{if } $\mathsf{T}$ is a leaf \textbf{then return} item at this leaf
    \State \quad $p_l \leftarrow \mathrm{ComputeMarginal}(\mathsf{T}_l, E , Y, \mathbf{G},\mathbf{Q})$
    \State \quad $p_r \leftarrow \mathrm{ComputeMarginal}(\mathsf{T}_r, E , Y, \mathbf{G},\mathbf{Q})$
    \State \quad \textbf{if } $\mathcal U(0,1) \le \frac{p_l}{p_l + p_r}$ \textbf{ then}
    \State \quad \quad \textbf{return } $\mathrm{SampleItem}(\mathsf{T}_l, E , Y, \mathbf{G},\mathbf{Q})$
    \State \quad \textbf{return } $\mathrm{SampleItem}(\mathsf{T}_r, E , Y, \mathbf{G},\mathbf{Q})$
    \State \textbf{procedure} $\mathrm{ComputeMarginal}(\mathsf{T}, E , Y, \mathbf{G},\mathbf{Q})$
    \State \quad $\Phi \leftarrow \mathbf{G}_{EY}^\top \mathsf{T}.\mathbf{A}_E \mathbf{G}_{EY} $
    \State \quad \textbf{return } $\mathbf{1}^\top \mathsf{T}.\mathbf{z}_E - \mathbf{1}^\top (\mathbf{Q} \circ \Phi) \mathbf{1}$
    \State \textbf{procedure} $\mathrm{ExtendInverse}(\mathbf{Q},\mathbf{u})$
    \State \quad \textbf{return} Inverse of $\mathbf{u}$ appended to $\mathbf{Q}^{-1}$
\end{algorithmic}
\end{algorithm}

More recently, many sophisticated sampling algorithms for DPPs have been explored and developed, extending the tree-based approach or otherwise. 
However, within the limited scope of this article we only address the most fundamental techniques, referring the interested reader to \cite{gartrell_scalablesampling, han2022scalable,kulesza-gillenwater, Tremblay-samplingdpp,anari-samplingdpp} and the references therein for recent developments.

%% file: montecarlo.tex
\subsection{Motivation}
Let $\mu$ be a finite positive Borel measure on $\mathbb R^d$.
A fundamental question in numerical analysis is to approximate integrals $\int_{\mathbb R^d} f(x) \mu(\d x)$, where $f$ ranges over some class of functions $\mathcal F$ on $\mathbb R^d$. 
Given $N \in \mathbb N$, a \textit{quadrature} is an algorithm that outputs $N$ points $x_1,\ldots,x_N\in \mathbb R^d$ (called \textit{nodes}) and weights $w_1,\ldots,w_N \in \mathbb R$. Based on these nodes and weights, we construct an estimator for $\int f(x) \mu(\d x)$ as
\[ \sum_{i=1}^N w_i f(x_i), \quad \text{ for each } f\in \mathcal F.\]
In general, the nodes and weights can depend on $N$ and $\mu$, but should not depend on any specific $f$. Given a quadrature, the approximation error is defined as
\[\mathcal E_N(f) := \sum_{i=1}^N w_i f(x_i) - \int f(x) \mu(\d x).\]
The asymptotic behavior of $\mathcal E_N(f)$ as $N\rightarrow\infty$  reflects the performance of the corresponding quadrature.

\subsection{Existing algorithms}
Given the classical nature of the problem, many quadrature methods have been developed, including Gaussian quadrature \cite{DavisRabi,Gautschi,BrassPetras, GautschiVarga}, Monte Carlo methods \cite{MonteCarlobook, bardenet:MCMC}, Quasi-Monte Carlo methods \cite{Dick_Pillichshammer_2010, Dick_2013, Oates16}, Monte Carlo with DPPs \cite{bardenet:montecarlogaussian, OPE-AOAP, bardenet:montecarlodpp, lemoine2024montecarlomethodscompact} and Bayesian quadrature \cite{Huszar12, Briol2015}, to name a few. In this survey, we briefly review Gaussian quadrature, Monte Carlo methods and Monte Carlo with DPPs.

\subsubsection{Gaussian quadrature}
We consider the case when $\mu$ is a measure supported on the interval $I:=[-1,1]$. Let $(\varphi_k)_{k\in\mathbb N}$ be a sequence of polynomials such that for each $k\in \mathbb N$, $\varphi_k$ has degree $k$, positive leading coefficient, and such that
\[ \int_I \varphi_k(x)\varphi_l(x)\mu(\d x) =\delta_{kl}, \quad\forall k,l \in \mathbb N.\]
Such a sequence is called \emph{orthonormal polynomials} associated with $\mu$.
We define the corresponding $N$-th Christoffel-Darboux kernel by 
\[ K_N(x,y) := \sum_{k=0}^{N-1} \varphi_k(x)\varphi_k(y).\]

Gaussian quadrature \cite{DavisRabi, Gautschi, BrassPetras} outputs $N$ nodes $x_1,\ldots,x_N$ to be the zeros of $\varphi_N(x)$ (which are real and simple) and weights $w_i = K_N(x_i,x_i)^{-1}$. The error is then given by
\[\mathcal E_N(f) = \sum_{i=1}^N \frac{f(x_i)}{K_N(x_i,x_i)} - \int f(x) \mu(\d x).\]

Gaussian quadrature is particularly effective when the test function $f$ is close to a polynomial. Specifically, this method gives $\mathcal E_N(f) =0$ for every polynomial $f$ of degree less than $2N-1$, and $\mathcal E_N(f)$ decays exponentially fast when $f$ is analytic (e.g., see \cite{GautschiVarga}). However, optimal rates of decay for $\mathcal E_N(f)$ are still unknown when $f$ is less regular. For instance, when $f$ is $C^1$, it is only known that $\mathcal E_N(f) = O(N^{-1})$, based on the classical Jackson's approximation theorem.

Another drawback of Gaussian quadrature is its limitation in higher-dimensional spaces. When the dimension $d\ge 2$, the zeros of multivariate polynomials form hypersurfaces, making it meaningless to use these zeros as nodes. Some modifications address this issue, such as considering $\mu = \otimes_{j=1}^d \mu_j$ as a product measure, where each $\mu_j$ is supported on $I$ and constructing a grid of nodes using $d$
one-dimensional Gaussian quadratures. However, this method suffers from the curse of dimensionality.

\subsubsection{Monte Carlo methods}
Monte Carlo methods \cite{MonteCarlobook} refers to picking up the nodes as the realization of random points in $\mathbb R^d$. A standard algorithm in this vein is the so-called \textit{importance sampling}. We consider the case when $\mu(\d x) = \omega(x) \d x$ supported in the hypercube $I^d$. Importance sampling outputs $N$ nodes $X_1,\ldots,X_N$ as i.i.d. samples from a \emph{proposal density} $q(x)\d x$ and 
\[ w_i = \frac{1}{N} \frac{\omega(X_i)}{q(X_i)} \cdot\]
The approximation error is then
\[\mathcal E_N(f) = \frac{1}{N}\sum_{i=1}^N \frac{\omega(X_i)}{q(X_i)} f(X_i) - \int_{I^d} f(x) \omega(x)\d x.\]

A simple calculation shows that $\mathbb E[\mathcal E_N(f)] = 0$.
If
\[ \sigma^2_f := \var_{X\sim q} \Big [\frac{\omega(X)f(X)}{q(X)}  \Big ] < \infty \]
then the classical CLT for independent random variables gives
\[  \sqrt{N} \mathcal E_N(f) \overset{law}{\longrightarrow} \mathcal N(0,\sigma_f^2).\]
In particular, $\var[\mathcal E_N(f)] = O(N^{-1})$.
Thus the typical order of magnitude of $\mathcal E_N(f)$ for the Monte Carlo method is $N^{-1/2}$,
which represents a slow decay rate for classical independent samplers. 

\subsection{Monte Carlo with DPPs}

In \cite{OPE-AOAP}, the authors show that using DPPs for Monte Carlo methods can give better decay rates for $\var[\mathcal E_N(f)]$. Their approach is based on the theory of (multivariate) orthogonal polynomials.

To elaborate, let $\mu$ be a measure supported in $I^d$.
We define an inner product associated with $\mu$
\[ \langle f,g \rangle_\mu := \int f(x)\overline{g(x)} \mu(\d x).\]
Consider the sequence of monomials $x_1^{k_1}\ldots x_d^{k_d}, k_i \in \mathbb N$, taken in the graded lexical order. Applying the Gram-Schmidt procedure to this sequence (w.r.t. $\langle \cdot,\cdot \rangle_\mu$), we obtain the orthonormal sequence $(\varphi_k)_{k\in \mathbb N}$ associated with $\mu$.
Let $K_N(x,y)$ be the corresponding $N$-th Christoffel-Darboux kernel for the (multivariate) orthogonal polynomial ensemble
\[K_N(x,y)= \sum_{k=0}^{N-1} \varphi_k(x)\varphi_k(y).\]

Let the nodes $\{X_1, \ldots, X_N \}$ be the DPP on $I^d$ with kernel $K_N$ and reference measure $\mu$, and let the weights be $w_i = K_N(X_i,X_i)^{-1}$. Then 
\begin{equation} \label{eq:MonteCarloDPP}
    \mathcal E_N(f) = \sum_{i=1}^N \frac{f(X_i)}{K_N(X_i,X_i)} - \int_{I^d} f(x) \mu(\d x)
\end{equation}
has mean zero by construction.
Intuitively, the negative dependence in DPPs results in more cancellation in linear statistics, often leading to smaller variance. Motivated by this intuition, the authors in \cite{OPE-AOAP} demonstrated that this is indeed the case for the multivariate OPE through a careful analysis of the estimator \eqref{eq:MonteCarloDPP}. To convey their idea while maintaining simplicity, we will present a particular case from their results.

\begin{theorem}[Theorem 1 in \cite{OPE-AOAP}]
    Let $\mu(\d x) = \d x$ be the uniform measure on the hypercube $I^d$. Then for any $f\in C^1$, compactly supported in $(-1,1)^d$, the estimator \eqref{eq:MonteCarloDPP} satisfies
    \[\var[\mathcal E_N(f)] = O \Big (\frac{1}{N^{1+1/d}}\Big ) \cdot \]
\end{theorem}
In fact, the authors further demonstrated that the rate mentioned above eventually applies to a larger class of measures $\mu$, and presented central limit theorems suited for Monte Carlo integration. For a more general result, we refer readers to the paper \cite{OPE-AOAP}. 


%% file: directionality.tex
\subsection{A parametric model}
 Directionality in data pertains to isotropy-breaking -- namely, the degree of dependency between points (for instance, repulsion) is much stronger in a few special directions compared to others. Thus, directionality corresponds to a type of (low dimensional) structure on data (compare, e.g., to sparsity), and is naturally applicable to learning scenarios such as clustering where such distinguished directions are of canonical importance. In \cite{GDP} the authors considered Gaussian Determinantal Processes (abbrv. GDP) as a well-structured parametric model of DPPs, with particular focus on exploring isotropy versus directionality in data.

For a $d \times d$ non-negative definite matrix $\Sigma$ (referred to as the scattering matrix), we define the  Gaussian Determinantal Process (GDP) as a translation-invariant determinantal process on $\R^d$ with kernel given by
$$
K(x-y)=\frac{1}{(2\pi)^{\frac{d}{2}} \sqrt{\det{\Sigma}}}\exp\big(-\frac{1}{2}(x-y)^\top \Sigma^{-1}(x-y)\big)\,, \quad x,y \in \R^d\,.
$$
and the standard Lebsegue measure on $\R^d$ as the background measure. It can be shown via Fourier analytic arguments that the Macchi-Soshnikov Theorem guarantees the existence of a GDP for any such scattering matrix $\Sigma$.

\subsection{A spiked model for DPPs and directionality in data }
The isotropic case of the GDP, where the dependency structure of the points is the same in all directions, corresponds to the choice $\Sigma = \sigma^2 I_d$, where $\sigma$ is a scalar and $I_d$ is the $d\times d$ identity matrix. For purposes of clarity, let us look at GDPs with average one particle per unit volume; this corresponds to \[K(0)=\frac{1}{(2\pi)^{\frac{d}{2}} \sqrt{\det{\Sigma}}}=1.\] Under this normalization, the truncated pair correlation function $\bar{\rho}_2(x,y)=-|K(x,y)|^2$ is given by \[\bar \rho_2(x,y)= - \exp\big(- (x-y)^T \Sigma^{-1} (x-y)  \big)\,.\] Therefore, if $x - y$ is well-aligned with an eigenvector of $\Sigma$ corresponding to a relatively large eigenvalue, then the magnitude of $\bar{\rho}_2(x, y)$ remains significant even when $\|x - y\|$ is large. As such, in these directions, the spatial correlation extends over a much longer range (compared to others).

This observation forms the basis of a spiked DPP model, analogous to the widely-used spiked covariance model \cite{BBP}, where the covariance matrix $\Sigma$ is modeled as a rank-one perturbation of the identity: $\Sigma = \text{Id} + \lambda uu^\top$, with $\|u\|=1$ and $\lambda \geq 0$. In this GDP setting, the directional structure is captured by the rank-one perturbation. In the context of GDPs, it is natural to consider scenarios where the presence of the spike does not alter the mean density of points compared to the isotropic case, where the scattering matrix is the identity. If this condition does not hold, the spike's presence can be detected by estimating the density. Maintaining the same mean density implies that $\det(\Sigma)$ is the same for both the null and alternative cases; this corresponds to shear-type transformations of the ambient Euclidean space.

Putting together the above observations, \cite{GDP} proposes the following spiked model for DPPs:
\begin{equation}
\label{eq:spike}
(2\pi) \cdot \Sigma=(1+\lambda)^{-\frac1{d-1}}(I_d-uu^\top) + (1+\lambda) uu^\top\,, \quad \|u\|=1\,,
\end{equation}
where $\lambda\ge 0$ is the strength parameter and $uu^\top$ is the spike. Notice that when $\lambda =0$, we are reduced to the isotropic setting $\Sigma=I_d$ (the factor $2 \pi$ is just a convenient normalization). 

\subsection{Inference on spiked models for GDP}
Consider a given collection of data points $\{X_1,\ldots,X_n\} \subset B_d(0;R)$ (where  $B_d(0;R)$ is the ball with centre 0 and radius $R$),  our goal is to infer distinguished directions along which there are long range dependencies in the data. We aim to do this in the framework of the GDP model. 

Inferential procedures on the GDP, as laid out in \cite{GDP}, are based on statistic $\hat \Sigma$ of  given by the $d \times d$ matrix
$$
\hat \Sigma := |B(0;1)|\frac{r_{n,d}^{d+2}}{d+2} I_d - \frac{1}{|B_d(0;R-r_{n,d})|}\sum_{i \in \mathcal{N}_0} \sum_{j \in \mathcal{N}_i} (X_i-X_j)(X_i-X_j)^\top ,
$$
where, for each $1\le i \le N, \, \mathcal{N}_i=\{j\,:\,j\ne i,  \,\|X_i-X_j\| < r_{n,d} \}$ is the  set of other observations that are at distance at most $r$ from it, and $\mathcal{N}_0=\{j\,:\, \|X_j\| < R-r_{n,d} \}$ is the set of \textit{interior} data points in $B(0;R)$. Here $|A|$ denotes the volume of a measurable set $A \subset \R^d$, and $r_{n,d}$ is a parameter chosen to be $r_{n,d}= c \cdot \sqrt{d \log n}$ for some constant $c$.

Setting \[\mathfrak{r}_{n,d}:=\frac{d^2(c \sqrt{\log n})^{d+1}}{\sqrt{n}}\] and denoting by $\chi(B)$ the indicator of the event $B$, it is demonstrated in \cite{GDP} that the test  
\[\psi_\delta=\chi\big\{ (2\pi)\|\hat \Sigma\|_{\mathrm{op}}>1+ \frac{\mathfrak{r}_{n,d}}{\delta}\big\}\]  
detects the presence of a spike with power $\delta$ as soon as the spike strength $\lambda$ is above the threshold $\lambda_{\mathrm{thresh}}=2\mathfrak{r}_{n,d}$. In the event that a spike is detected, the spike direction can be effectively estimated by the maximal eigenvector of $\hat \Sigma$. It has been demonstrated in \cite{GDP} that these statistical procedures achieve parametric rates in fixed dimensions. This raises natural and interesting questions vis-a-vis the study of GDP type models and related statistical procedures in a high dimensional setting, wherein new kinds of random matrix ensembles are expected to emerge.

\subsection{Clustering and random matrices}
A key application of GDP based methodology, as outlined above, is to clustering problems, and more generally, the problem of dimensionality reduction. A natural comparison would be to classical dimension reduction methods based on Principal Component Analysis (PCA), wherein low dimensional approximation to data is achieved via projections onto the principal components (i.e., principal eigendirections) of the sample covariance matrix. 

It is well known that the principal directions in PCA correspond to the directions of maximal variance in the data. In \cite{GDP}, dimension reduction methodology based on finding directions of maximal repulsion was proposed, based on the GDP as an ansatz. In this vein, we obtain dimension reduction by projecting onto the principal eigendirections of the matrix $\hat \Sigma$ as above. The performance of this approach has been demonstrated on the celebrated Fisher's Iris dataset, which is a prototypical benchmark for clustered data \cite{GDP}. 

\begin{figure}[h]
    \centering
    \includegraphics[width=0.49\linewidth]{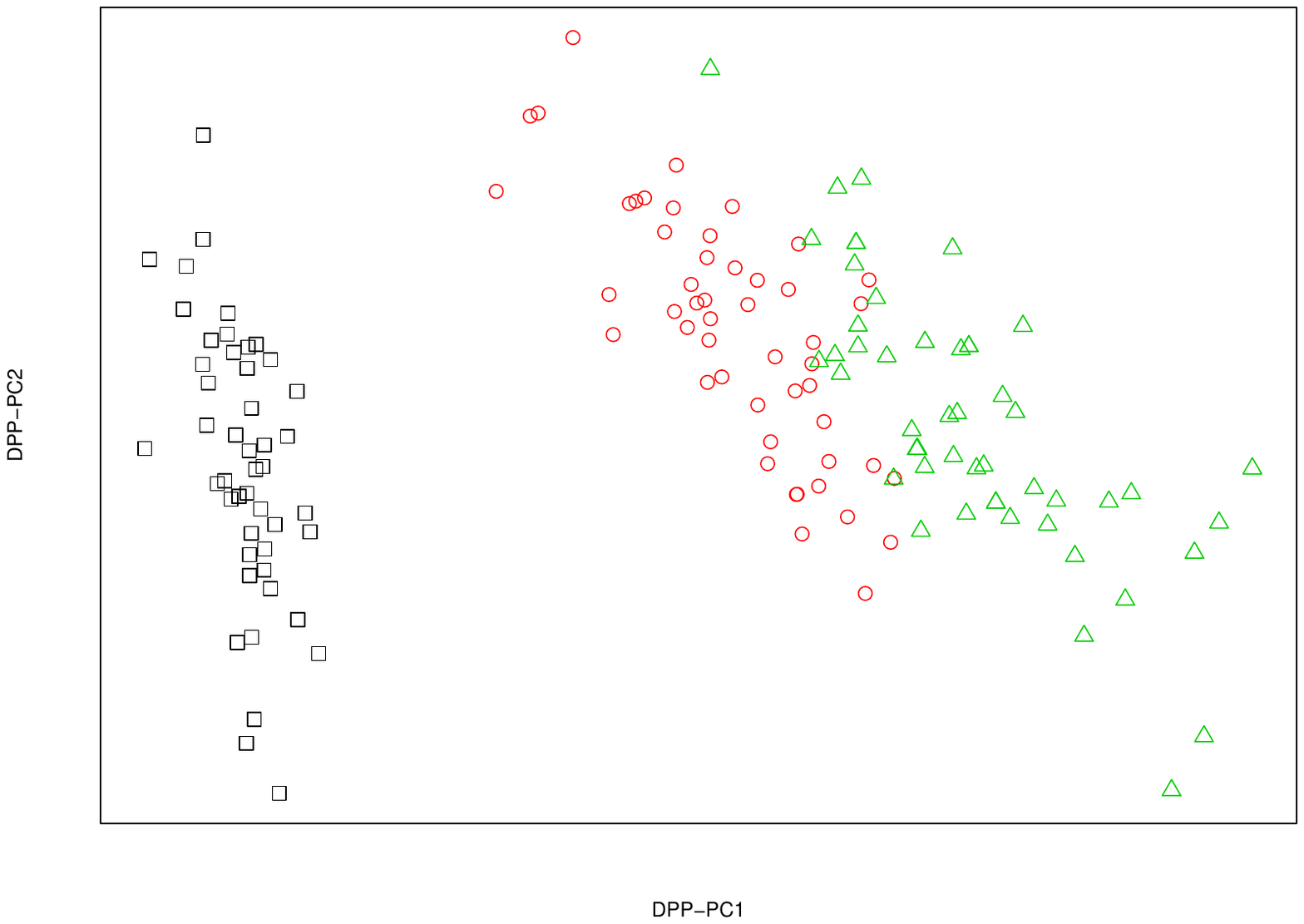}
    \includegraphics[width=0.49\linewidth]{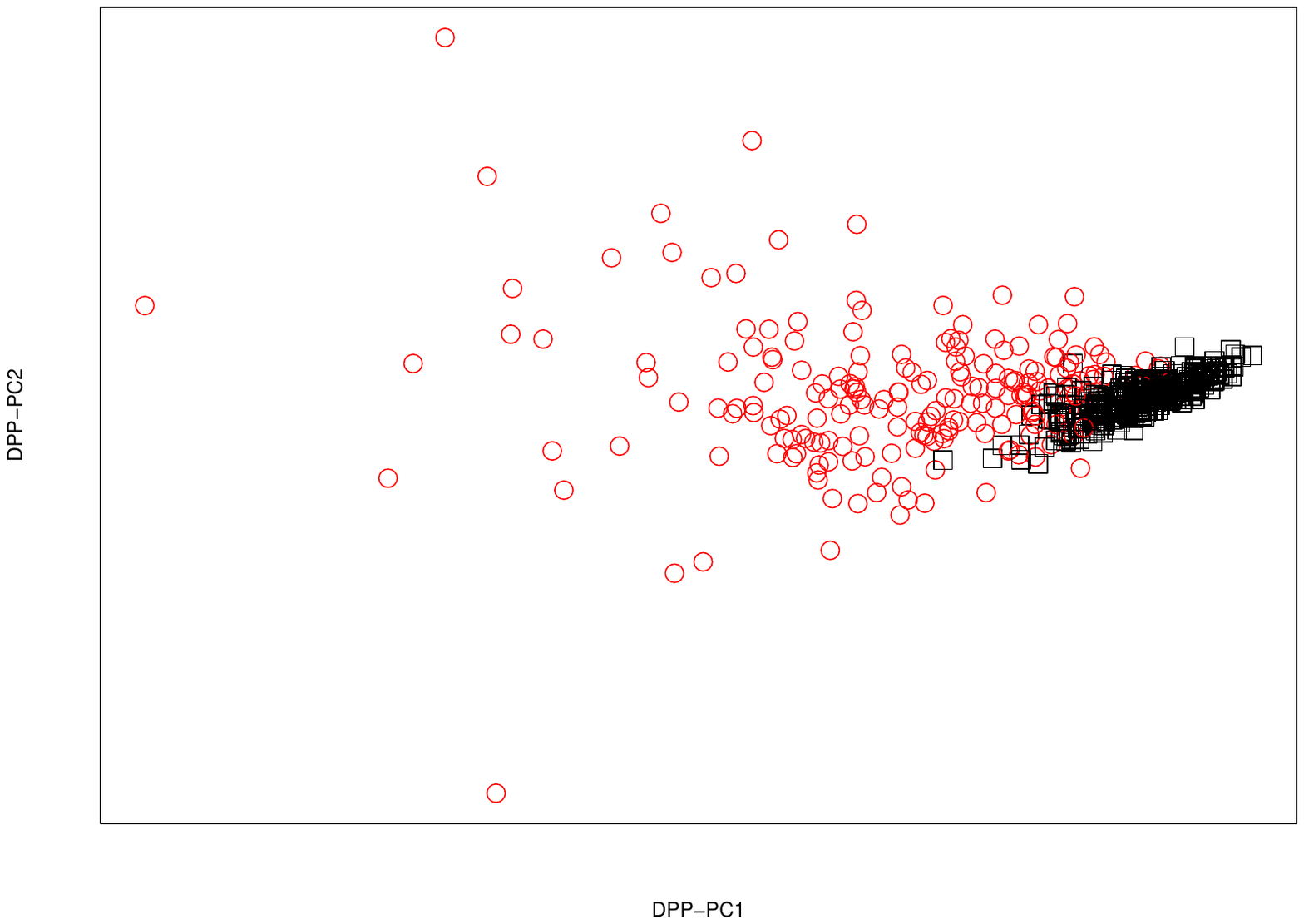}
    \caption{Left: Iris-DPP \qquad \qquad Right: W-DPP \qquad (c.f. \cite{GDP})}
    \label{fig:iris}
\end{figure}

It may be noted that the essential part of the estimator $\hat \Sigma$ which determines its eigendirections (thereby playing the central role in clustering and dimension reduction methods based on it) is the component 
\[ \tilde \Sigma = \sum_{i \in \mathcal{N}_0} \sum_{j \in \mathcal{N}_i} (X_i-X_j)(X_i-X_j)^\top .\]

Since the data points $\{X_1,\ldots,X_n\}$ are random, $\tilde \Sigma$ is a random matrix. The GDP model, and its applications to clustering and dimension reduction, therefore naturally leads to intriguing questions in random matrix theory. In particular, the spiked model \eqref{eq:spike} leads to a new kind of spiked random matrix model featuring a kernelized truncation (based on mutual separation).

It maybe noted that, even if the data points $\{X_1,\ldots,X_n\}$ are independent, random matrices of the form $\tilde \Sigma$ have dependent entries (due to the truncation in the definition). However, as a random matrix model, it still offers a significant degree of structure so that hope for analytical progress is still reasonable. Indeed, in \cite{ghosh-mukherjee-talukdar}, the authors investigate the most basic setting for such kernelized random matrix ensembles, where the points $\{X_1,\ldots,X_n\}$ are i.i.d., and establish the emergence of novel families limiting spectral distributions that generalize the classical Marcenko-Pastur laws, and beyond. The methods involve crucial ingredients from free probability theory. We expect that, motivated by the GDP ansatz and its methodological application, the study of such kernelized ensembles will lead to novel and interesting problems in random matrix theory with dependent entries.


%% file: coreset.tex
In recent years, DPPs have emerged as a powerful sampling technique for various machine learning applications. This includes, but is not limited to, applications of foundational importance such as Stochastic Gradient Descent (abbrv. SGD), the general problem of sampling coresets, sampling for recommender systems, to name a few. 

Roughly speaking, the diversity accorded by a determinantal sample leads to stable statistical outcomes, leading to downstream improvements in performance guarantees and convergence rates in a variety of applications. In addition to applications, this approach gives rise to several natural and interesting mathematical problems related to stochastic properties of DPPs (such as concentration phenomena), some of which have seen activity and progress in the recent past.

In this section, we undertake a discussion of DPPs as a sampling device, and the mathematics and the applications surrounding this paradigm.

\subsection{Coresets in machine learning}


In machine learning, learning tasks are often formulated as an optimization problem of minimizing a suitable loss function \cite{Shalev-Shwartz_Ben-David_2014}. Let $\mathcal X$ be a set of size $N$ (called the data set) and $\mathcal F$ be a family of real-valued functions on $\mathcal X$. In this survey, we will focus on a loss function of the following form
\begin{equation} \label{eq:loss}
    L(f) = \sum_{x\in \mathcal X} f(x), \quad f \in \mathcal F.
\end{equation}
This setting includes many classical learning tasks. Here are some examples.

\begin{example}[$k$-means] \label{eg:k-mean}
    Let $k\in \mathbb N$ and $\mathcal F$ contain functions of the following form
    \[ f_{\mathcal C}(x) = \min_{q\in \mathcal C} \|x-q\|_2^2,\]
    where $\mathcal C$ is a set of $k$ points (called \textit{cluster centers}) in $\mathbb R^d$. The $k$-means problem refers to finding a cluster $\mathcal C^*$ that minimizes \eqref{eq:loss}.
\end{example}

\begin{example}[linear regression] \label{ex:linear}
    Let $\mathcal X = \{x_i = (y_i,z_i) \in \mathbb R^d\times \mathbb R: 1\le i \le N \} $. Linear regression corresponds to minimizing \eqref{eq:loss} over 
    \[ \mathcal F = \{(y,z)\mapsto |\langle a, y \rangle + b - z|^2 , a \in \mathbb R^d, b \in \mathbb R\}.\]
\end{example}

When the data size $N$ is large, the computational cost can become prohibitive, and the optimization problem often becomes intractable. A natural approach is to reduce the amount of data by using only a manageable number of representative samples. This idea is formalised by the notion of \textit{coreset} (see \cite{bachem2017coresetML}).
\begin{definition}[Coresets]
    Let $\mathcal S$ be a subset of $\mathcal X$ and $\{w(x),x\in \mathcal S\}$ be some weights associated with $\mathcal S$. We define
    \[L_{\cS}(f):= \sum_{x\in \cS} w(x) f(x).\]
    Given $\varepsilon>0$,  $(\cS,w)$ is called an $\varepsilon$-coreset for $\mathcal F$ if
    \[ \Big |\frac{L_{\mathcal S}(f)}{L(f)} - 1 \Big | \le \varepsilon, \quad \text{ uniformly in } f\in \mathcal F.\]
\end{definition}
The existence of coresets is trivial since $\mathcal X$ with weights $1$ is always a coreset for any $\mathcal F$. In fact, only coresets of small size are interesting. We will focus on constructing coresets of a given size $m \ll N$, for which we will refer to as \textit{coreset problem}.

Constructing deterministic coresets is a challenge, and practical coresets are often specially designed for each particular learning tasks. In fact, practitioners prefer \emph{randomized} coresets, namely, a random subset $\mathcal S$ of $\mathcal X$ being a coreset with high probability.  When $\mathcal S$ is random, a natural choice for the weights is $w(x) = \mathbb P(x\in \mathcal S)^{-1}$. This particular choice makes $L_{\mathcal S}(f)$ an unbiased estimator of $L(f)$ 
\begin{eqnarray*}
\mathbb E[L_{\mathcal S}(f)] = \mathbb E \Big [ \sum_{x\in \mathcal S} \frac{f(x)}{\mathbb P(x\in \mathcal S)} \Big ] 
=
\mathbb E \Big [ \sum_{x\in \mathcal X} \frac{f(x)}{\mathbb P(x\in \mathcal S)} \mathbf{1}_{x\in \mathcal S}\Big ] 
= \sum_{x\in \mathcal X} \frac{f(x)}{\mathbb P(x\in \mathcal S)}\mathbb E[\mathbf{1}_{x\in \mathcal S}] = L(f).    
\end{eqnarray*}

Let $\mathcal S$ be a random subset of $\mathcal X$ and define
\[ L_\cS(f) := \sum_{x\in \cS} \frac{f(x)}{\mathbb P(x\in \cS)}, \quad f\in \mathcal F.\]
The event that $\mathcal S$ is an $\varepsilon$-coreset for $\mathcal F$ can be rewritten as
\begin{equation} \label{eq:coreset}
    \Big \{ \Big |\frac{L_{\mathcal S}(f) - \mathbb E[L_{\mathcal S}(f)]}{\mathbb E[L_\cS(f)]} \Big | \le \varepsilon , \: \forall f\in \mathcal F \Big \}.
\end{equation}
Hence, for $\cS$ to be an $\varepsilon$-coreset of $\mathcal F$ (with high probability), one needs each $L_\cS(f)$ to highly concentrate around its mean, and then performs a chaining argument to obtain an uniform bound over $\mathcal F$.  
Thus, constructing coresets will require \textit{concentration inequalities} for each $L_{\mathcal S}(f)$ and techniques from \textit{chaining}.

\subsection{DPPs for minibatch sampling in SGD}
Before presenting results on DPP-based coresets, we briefly discuss minibatch sampling in stochastic gradient descent (SGD), another important topic in machine learning which turns out to be closely related to our coreset problem. 

Consider the problem of minimizing an empirical loss
\[ \min_{\theta \in \Theta} \frac{1}{N} \sum_{i=1}^N \ell_{\theta}(x_i),\]
where $\mathcal X:= \{x_1,\ldots,x_N\}$ is the data set, and the loss $\ell_{\theta}(\cdot)$ is differentiable in the parameter $\theta$. When $N\gg 1$, the computational cost for the gradient of the objective function becomes expensive. SGD method \cite{Robbins_SGD} refers to building an estimator for the gradient at each iteration of the gradient descent, using a small samples of data points (called a \textit{minibatch}). Given $m\ll N$, a key problem in SGD is to construct minibatches of size $m$ that keep the variance of the gradient estimators as small as possible \cite{NIPS2011_SGD, zhang2017_SGD}. More precisely, we want to construct a (random) subset $\mathcal S \subset \mathcal X$ of size $m$, such that
\[ L_\cS(\theta) := \sum_{x\in \mathcal S} w(x) \nabla_{\theta} \ell_{\theta}(x) \approx \frac{1}{N} \sum_{i=1}^N \nabla_{\theta} \ell_{\theta}(x_i) =: L(\theta),\]
and $\mathbb E|L_\cS(\theta) - L(\theta)|^2$ is small, uniformly in $\theta\in \Theta$. From this point of view, the problem of minibatch sampling in SGD has a similar flavor to the coreset problem we introduced above.

In \cite{OPE-NIPS}, the authors used DPPs to construct minibatch sampling for SGD which yields significant variance reduction, compared to independent sampling. To elaborate, we assume that the data $x_i$'s are drawn i.i.d. from a distribution $\gamma(x)\d x$ on the hypercube $[-1,1]^d$. Using a suitable multivariate OPE and spectral approximation, Bardenet et al. constructed an $N\times N$ orthogonal projection matrix $\mathbf{K}$ of rank $m$. Let $\mathcal S$ be the DPP on $\mathcal X$ defined by the kernel matrix $\mathbf K$ w.r.t. the reference measure $N^{-1} \sum_{i=1}^N \delta_{x_i}$, the estimator is then given by
\[ L_{\mathcal S_\dpp}(\theta) := \sum_{x_i \in \mathcal S} \frac{1}{\mathbf K_{ii}} \nabla_\theta \ell_{\theta}(x_i). \]
Under mild assumptions, they proved that
\begin{theorem}
    With high probability in the data set, we have
    \[ \mathbb E [L_{\mathcal S_\dpp}(\theta) | \mathcal X] = L(\theta) \quad, \quad \var[L_{\mathcal S_\dpp}(\theta)|\mathcal X] = O_P(m^{-(1+1/d)}), \quad \forall \theta \in \Theta. \]
\end{theorem}
For more detailed discussions, we refer the readers to their paper \cite{OPE-NIPS} and the references therein. We remark that, for independent sampling minibatches of size $m$, the variance of $L_\cS(\theta)$ is of order $m^{-1}$. Thus, sampling with DPPs improves the decay rate of the variance in the exponent, which is significant.

\subsection{Importance sampling for coresets}
We come back to the problem of randomized coresets construction. 
A well-studied method to construct coresets is the so-called \emph{importance sampling} (\cite{bachem2017coresetML}, Chapter 2). Given a proposal probability distribution $q$ on $\mathcal X$ and $m\in \mathbb N$, importance sampling outputs $\mathcal S_\iid = \{X_1,\ldots,X_m\}$, where $X_i \sim_\iid q$, and weights $w(x)=(mq(x))^{-1}$. Thus, 
\begin{equation} \label{eq:coreset_iid}
    L_{\mathcal S_\iid}(f) = \frac{1}{m} \sum_{i=1}^m \frac{f(X_i)}{q(X_i)}\cdot
\end{equation}

It is easy to see that 
\[\mathbb E[L_{\mathcal S_\iid}(f)] = \mathbb E \Big [ \frac{1}{m} \sum_{i=1}^m \frac{f(X_i)}{q(X_i)}\Big ] = \mathbb E \Big [ \frac{f(X_1)}{q(X_1)} \Big ] = \sum_{x\in \mathcal X} \frac{f(x)}{q(x)}q(x) = L(f),\]
which implies $L_{\mathcal S_\iid}(f)$ is an unbiased estimator of $L(f)$, for every $f\in \mathcal F$. Moreover,
\begin{equation} \label{eq:var_iid}
    \var[L_{\mathcal S_\iid}(f)] = \frac{1}{m} \var \Big [ \frac{f(X_1)}{q(X_1)}\Big ] = \frac{1}{m} \sum_{x\in \mathcal X} \Big | \frac{f(x)}{q(x)} - L(f) \Big |^2 q(x). 
\end{equation}

A natural task is then to find a distribution $q$ that makes $\var[L_{\mathcal S_\iid}(f)]$ small uniformly in $f\in \mathcal F$. A candidate for this task is the so-called \textit{sensitivity based importance sampling}, introduced by Langberg and Schulman \cite{sensitivity}. For simplicity, let us consider the case when $\mathcal F$ consists of non-negative functions.

\begin{definition}
    The sensitivity of a point $x \in \mathcal X$ with respect to $\mathcal F$ is defined as
    \[ \sigma(x) := \sup_{f\in \mathcal F} \frac{f(x)}{L(f)} \cdot\]
    The total sensitivity is defined as $\mathfrak S := \sum_{x\in \mathcal X} \sigma (x)$.
\end{definition}

The intuition behind the notion of sensitivity is that those points which contribute more to the loss should be more likely to be sampled.
In practice, computing sensitivity could be a challenge; in fact, we often only have access to upper bounds of sensitivity, namely, a function $s:\mathcal X \rightarrow \mathbb R_+$ such that $s(x) \ge \sigma(x)$ for every $x\in \mathcal X$. A \emph{sensitivity based importance sampling scheme} is to sample $X_i$ i.i.d. from 
$q(x) := s(x)/S$, where $S:= \sum_{x\in \mathcal X} s(x).$ From \eqref{eq:coreset_iid}, we have
\begin{equation}\label{eq:importantsamplingcoreset}
    \frac{L_{\cS_\iid}(f)}{L(f)} = \frac{1}{m}\sum_{i=1}^m \frac{f(X_i)}{L(f) q(X_i)}\cdot
\end{equation}
From the definition of sensitivity and the non-negativity of $f$, we have
    \[0 \le \frac{f(x)}{L(f)q(x)} = S\frac{f(x)}{L(f) s(x)} \le S, \quad \forall x\in \mathcal  X.\]
Applying the classical Hoeffding's inequality for the sum of independent bounded random variables in \eqref{eq:importantsamplingcoreset} gives

\begin{proposition} \label{p:concentration-iid}
    For any $\varepsilon>0$ and for any $f\in \mathcal F$, we have
    \[ \mathbb P \Big ( \Big |\frac{L_{\mathcal S_\iid}(f)}{L(f)} -1 \Big | > \varepsilon \Big ) \le 2 \exp \Big ( -\frac{2m \varepsilon^2}{S} \Big ) \cdot \]
\end{proposition}

The sharper the bound $s(x) \ge \sigma(x)$, the better performance of importance sampling. Obtaining a sharp upper bound on sensitivity could be a challenge and will depend on the specific models (see \cite{bachem2017coresetML}). 

After establishing the concentration result for each  $L_\cS(f)$,  the next step is to obtain a uniform bound over $\mathcal F$.  For this task, it is necessary to quantify the complexity of $\mathcal F$, and the uniform bound will depend heavily on that quantity. In \cite{bachem2017coresetML}, the authors use \emph{Vapnik-Chervonenkis dimension} \cite{VCdim} and \emph{pseudo-dimension} to quantify the complexity of $\mathcal F$, and then apply a chaining argument to obtain the desired upper bound. However, it is important to note that this chaining argument relies on the i.i.d. nature of the sampling scheme. For more details, we refer the readers to \cite{bachem2017coresetML}.

\subsection{DPPs as coresets}

 Intuitively, DPPs avoid situations where a data point is sampled multiple times, leading to redundancy. Furthermore, as we have already observed, DPPs with suitable kernels often result in smaller variance for linear statistics (\cite{OPE-AOAP, OPE-NIPS}). 
Thus, there is hope to improve performance by using DPPs instead of independent sampling.

In \cite{tremblay2019determinantal}, Tremblay et al. provide extensive theoretical and empirical justification for the use of DPPs to construct randomized coresets.  However, the crucial question of whether DPP-based coresets can provide a strict improvement remained open. As mentioned earlier, constructing coresets necessitates concentration inequalities and chaining, and these theories have been well-developed for independent processes. Unfortunately, corresponding results for correlated processes are less established, which could pose an obstacle for sampling with correlated processes.

In a recent result \cite{NIPS2024}, the authors confirm that, with a careful choice of kernels, DPP-based coresets can outperform independently drawn coresets. The key ingredient is a Bernstein-type concentration inequality for linear statistics of DPPs. This result was first investigated in a seminal paper \cite{BREUER2014441} by Breuer and Duits for projection DPPs. In \cite{NIPS2024}, the authors extend this result to the case of DPPs with general Hermitian kernels, and even non-symmetric kernels, which seems to go beyond the state of the art. It is worth noting that DPPs with non-symmetric kernels have recently been shown to be of significant interest in machine learning, but come with a limited theoretical toolbox \cite{gartrell2019learning, gartrell2020scalable, han2022scalable}.

\subsubsection{Concentration inequalities for linear statistics of DPPs}

In \cite{PEMANTLE_PERES_2014}, Pemantle and Peres introduced concentration inequalities for Lipschitz functionals of probability measures on $2^{\mathcal X}$ (equivalently, random subsets of $\mathcal X$) satisfying the \textit{stochastic covering property} (SCP). 
We refer the readers to \cite{PEMANTLE_PERES_2014} to the precise definition of SCP. However, we remark that the class of probability measures on $2^{\mathcal X}$ satisfying the SCP particularly includes DPPs.

For simplicity, we only state here the result for $m$-homogeneous point processes (i.e., point processes with a.s. $m$ points). For the general case, we refer the readers to \cite{PEMANTLE_PERES_2014}.
\begin{theorem}[Theorem 3.1 in \cite{PEMANTLE_PERES_2014}] \label{t:concentration-lips}
    Let $\mathcal S$ be an $m$-homogeneous point process on $\mathcal X$ satisfying the SCP. Let $F$ be a $1$-Lipschitz function (with respect to the total variation distance) on counting measures with total mass $m$ on $\mathcal X$. Then
    \[\mathbb P(F(\mathcal S) - \mathbb E[F(\mathcal S)] \ge a ) \le \exp \Big (-\frac{a^2}{8m}\Big), \quad \forall a \ge 0.\]
\end{theorem}

This concentration inequality has a similar flavor to the classical Hoeffding inequality, in the sense that it provides a Gaussian-tail bound for $F(\mathcal S) - \mathbb E[F(\mathcal S)]$ but does not take the variance into account. 

Since we are particularly interested in linear statistics of DPPs, we will apply Theorem \ref{t:concentration-lips} to the particular case where $\mathcal S$ is a projection DPP of rank $m$, and $F(\mathcal S) = \sum_{x\in \mathcal S} f(x) = \Lambda_{\mathcal S}(f)$. It is easy to see that if $\|f\|_\infty := \sup_{x\in \mathcal X}|f(x)|\le 1$ then $\Lambda_{\mathcal S}(f)$ is $1$-Lipschitz.

\begin{corollary} \label{c:concentration-lips}
    Let $\mathcal S$ be a projection DPP on $\mathcal X$ of rank $m$ and let $f:\mathcal X\rightarrow \mathbb R$ be such that $\|f\|_\infty \le 1$. Then
    \[\mathbb P(\Lambda_{\mathcal S}(f) - \mathbb E[\Lambda_{\mathcal S}(f)] \ge a ) \le \exp \Big (-\frac{a^2}{8m \|f\|_\infty^2}\Big), \quad \forall a \ge 0.\]
\end{corollary}

Let us apply Corollary \ref{c:concentration-lips} to the setting of coresets. We remark that $L_{\mathcal S}(f)$ could be realized as a linear statistic
\begin{equation} \label{eq:li-stat}
    L_{\mathcal S}(f) = \sum_{x\in \mathcal S} \frac{f(x)}{\mathbb P(x\in \mathcal S)} 
    = \sum_{x\in \mathcal S} \frac{f(x)}{K(x,x)} 
    = \Lambda_{\mathcal S}\Big ( \frac{f(\cdot)}{K(\cdot,\cdot)}\Big ) ,
\end{equation}
where we used the fact that if $\cS$ is a DPP with kernel $K$ and the background measure is the counting measure on $\mathcal X = \{x_1,\ldots,x_N\}$, then $\mathbb P(x\in \cS) = K(x,x), \: \forall x\in \mathcal X$.
Using \eqref{eq:li-stat}, one has
\begin{equation} \label{eq:coreset-linearstat}
    \mathbb P \Big (\Big |\frac{L_{\mathcal S}(f)}{L(f)}-1\Big | \ge \varepsilon \Big ) 
= \mathbb P \Big (\Big | \Lambda_{\mathcal S}\Big (\frac{f(\cdot)}{L(f) K(\cdot,\cdot)} \Big ) - \mathbb E \Lambda_{\mathcal S}\Big (\frac{f(\cdot)}{L(f)K(\cdot,\cdot)} \Big ) \Big | \ge \varepsilon \Big ) \cdot
\end{equation}
If $|\mathcal S| = m$ a.s., then $K(x,x)=\mathbb P(x \in \mathcal S)$ is typically $m/N$. For a bounded non-negative function $f$, $L(f)$ is of order $N$. Thus,
\[ \Big \| \frac{f(\cdot)}{L(f) K(\cdot,\cdot)}  \Big \|_\infty \sim \frac{1}{m} \cdot \]
Applying Corollary \ref{c:concentration-lips} will give a concentration inequality of the form
\begin{equation} \label{eq:concentration-coreset-iid}
    \mathbb P \Big (\Big |\frac{L_{\mathcal S}(f)}{L(f)}-1\Big | \ge \varepsilon \Big ) \le 2 \exp (-Cm\varepsilon^2 ),
\end{equation}
for some constant $C>0$ (possibly depending on $f$), which is not better than independent sampling (c.f. Proposition \ref{p:concentration-iid}). One possible reason is that the concentration inequalities in Theorem \ref{t:concentration-lips} and Proposition \ref{c:concentration-lips} do not take the variance into account; meanwhile the variance reduction for linear statistics of DPPs is one of the main motivation for sampling with DPPs.
Thus, a concentration inequality for linear statistics of DPPs which takes variance into account is a desire. As we introduced earlier, such results have been investigated in \cite{BREUER2014441, NIPS2024}. 
\begin{theorem}[Hermitian kernels]\label{t:concentration-dpp}
    Let $\mathcal S$ be a DPP on $\mathcal X$ with a Hermitian kernel $K$. Let $f:\mathcal X \rightarrow \mathbb R$ be a bounded test function, then
    \[ \mathbb P(|\Lambda_{\mathcal S}(f) - \mathbb E \Lambda_{\mathcal S}(f) | \ge a ) \le 2 \exp \Big ( - \frac{a^2}{4A\var[\Lambda_{\mathcal S}(f)]} \Big ), \quad \forall 0 \le a \le \frac{2A \var[\Lambda_{\mathcal S}(f)]}{3\|f\|_\infty}, \]
    where $A>0$ is an universal constant.
\end{theorem}

\begin{theorem}[Non-symmetric kernels] \label{t:concentration-dpp-nonsym}
    Let $\mathcal S$ be a DPP on $\mathcal X$ with a non-symmetric kernel $K$. Let $f: \mathcal X \rightarrow \mathbb R$ be a bounded test function, then
    \[\mathbb P(|\Lambda_{\mathcal S}(f) - \mathbb E[\Lambda_{\mathcal S}(f)]| \ge a ) 
    \le 2 \exp \Big ( -{a^2 \over 4\var[\Lambda_{\mathcal S}(f)]} \Big ), 
    \: \forall 0 \le a \le {\var[\Lambda_{\mathcal S}(f)]^2 \over  40\|f\|_\infty^3 \cdot \max(1,\|K\|_\op^2) \cdot  \|K\|_* } ,\]
    where $\|\cdot\|_\op$ denotes the spectral norm and $\|\cdot \|_*$ denotes the nuclear norm of a matrix.
\end{theorem}

\begin{remark}
    For simplicity, we will restrict our attention to the case of Hermitian kernels. However, we remark that we always can use Theorem \ref{t:concentration-dpp-nonsym} to obtain analogous results in the non-symmetric case.
\end{remark}

Applying Theorem \ref{t:concentration-dpp} to Equation \eqref{eq:coreset-linearstat} gives 
\begin{corollary} \label{c:coreset-dpp}
      Let $\mathcal S$ be a DPP on $\mathcal X$ with a Hermitian kernel $K$ and $f:\mathcal X \rightarrow \mathbb R$ be a bounded measurable function. Then
      \[ \mathbb P \Big ( \Big |\frac{L_{\mathcal S}(f)}{L(f)} -1 \Big | \ge \varepsilon \Big ) \le 2 \exp \Big ( -\frac{\varepsilon^2}{4A} \var \Big [ \frac{L_{\mathcal S}(f)}{L(f)}\Big ]^{-1} \Big ),\]
      for $0 \le \varepsilon\le \frac{2A}{3} \cdot \var \Big [ \frac{L_{\mathcal S}(f)}{L(f)}\Big ] \cdot 
      \Big \| \frac{f(\cdot)}{L(f) K(\cdot,\cdot)} \Big \|_\infty^{-1}$.
\end{corollary}

To compare, assume that $\mathbb E[|\mathcal S|]=m$, then typically $\var[ L_{\mathcal S}(f)/L(f)] = O(m^{-(1+\delta)})$,  for some parameter $\delta>0$, depending on the specific model (see \cite{OPE-NIPS}, \cite{OPE-AOAP}). Hence, Corollary \ref{c:coreset-dpp} yields a concentration inequality of the form
\begin{equation} \label{eq:concentration-coreset-dpp}
    \mathbb P \Big ( \Big |\frac{L_{\mathcal S}(f)}{L(f)} -1 \Big | \ge \varepsilon \Big ) \le 2 \exp ( -C m^{1+\delta}\varepsilon^2  ),
\end{equation}
for some constant $C>0$, which is clearly tighter than \eqref{eq:concentration-coreset-iid}.

A limitation of Corollary \ref{c:coreset-dpp} is the range of $\varepsilon$ for which the concentration inequality is valid. Extending this result to all non-negative reals $\varepsilon$ is a desire. However, this task is really challenging and remains open.

\subsubsection{PAC bound}
As we remarked earlier, the uniform bound over $\mathcal F$ will particularly depend on the complexity of the class $\mathcal F$; thus, an universal chaining argument for all cases could be a challenge. We will restrict ourselves to some common functional classes in machine learning.

Before going to details, let us assume here an assumption, which is not harmful due to homogeneity (namely, $L_\cS(f)/L(f) = L_\cS(\lambda f)/L(\lambda f)$ for any $\lambda \neq 0$). However, this assumption will be useful later.
\begin{equation} \tag{A.0} \label{eq:lowerbound}
   \frac{1}{N}|L(f)| \ge c , \quad \text{ for some } c>0 \text{ uniformly over } \mathcal F.
\end{equation}

The first scenario is when $\mathcal F$ is a subset of a finite dimensional vector space of functions. In particular, we assume that
\begin{equation} \tag{A.1} \label{eq:finite-dim}
    \dim \span_{\mathbb R}(\mathcal F) = D < \infty.
\end{equation}
This assumption covers common situations like linear regression in Example \ref{ex:linear}, where we observe that each $f\in \mathcal F$ is a quadratic function in $(d+1)$ variables. 
Thus the dimension of the linear span of $\mathcal F$ is at most $(d+1)^2 + (d+1) + 1$. 
Another popular class is the class of \emph{band-limited functions}, originating in signal processing problems.
A function $f : \mathbb{T}^d \mapsto \mathbb{R}$ (where  $\mathbb T^d$ denotes the $d$-dimensional torus) is said to be \textit{band-limited} if there exists $B \in \mathbb N$ such that its Fourier coefficients $\hat f (k_1,\ldots,k_d) =0$ whenever there is a $k_j$ such that $|k_j| > B$. 
It is easy to see that the space of $B$-bandlimited functions is of the dimension at most $(2B+1)^d$.

Another common scenario is when $\mathcal F$ is parametrized by a finite-dimensional parameter space
\begin{equation} \tag{A.2} \label{eq:parametrize}
    \mathcal F = \{f_\theta: \theta \in \Theta\}, \text{ where $\Theta$ is a bounded subset of $\mathbb R^D$ for some $D$}.
\end{equation}
We further assume a Lipschitz condition on $\mathcal F$
\begin{equation} \tag{A.3} \label{eq:lipschitz}
   \|f_{\theta} - f_{\theta'}\|_\infty \le \ell \|\theta - \theta'\| \text{ for some $\ell>0$, uniformly on $\Theta$.}
\end{equation}
These conditions particularly cover the $k$-means problem of Example~\ref{eg:k-mean}, as well as linear/non-linear regression settings. 
For instance, for $k$-mean, each query is parametrized by its cluster centers $\mathcal C= \{q_1,\ldots,q_k\}$, which can be viewed as a parameter $(q_1,\ldots,q_k)$ in $ \mathbb R^{kd}$.  

In these two common situations, we can perform a chaining argument to obtain the following result (see \cite{NIPS2024}).

\begin{theorem}\label{t:core-union}
    Let $\mathcal S$ be a DPP with a Hermitian kernel $\bK$ on a finite set $\mathcal X = \{x_1, \dots, x_N\}$ and $m = \mathbb E[|\mathcal S|]$. Assume that for all $i\in\{1, \dots, N\}$, $\bK_{ii} \ge \rho \cdot m/N$ for some $\rho> 0$ not depending on $m,N$. 
    Let $V\ge \sup_{f\in \mathcal F} \Var[N^{-1} L_{\mathcal S}(f)]$. 
    \medskip
    
    Assuming \eqref{eq:lowerbound} and \eqref{eq:finite-dim}, we have
   \[\mathbb P
   \Big ( \exists f \in \mathcal F: 
  \Big |\frac{ L_{\mathcal S}(f)}{L(f)}-1 \Big | \ge \varepsilon \Big ) \le 2 \exp \Big (6D - {c^2 \varepsilon^2 \over 16 AV} \Big ), \quad 0 \le  \varepsilon \le  {4A\rho m V \over 3c \sup_{f\in \mathcal F}\|f\|_\infty} \cdot \]

  Assuming \eqref{eq:lowerbound}, \eqref{eq:parametrize}, \eqref{eq:lipschitz} and $|\mathcal S| \le B \cdot m$ a.s. for some $B>0$, we have
   \[\mathbb P
   \Big ( \exists f \in \mathcal F: 
  \Big |\frac{L_{\mathcal S}(f)}{L(f)}-1 \Big | \ge \varepsilon \Big ) \le  2 \exp \Big (CD- D\log \varepsilon - {c^2 \varepsilon^2 \over 16 AV}\Big ), \: 0 \le  \varepsilon \le  {4A\rho mV \over 3c \sup_{f\in \mathcal F}\|f\|_\infty} \cdot \]
  Here $A>0$ is a universal constant and $C=C(\Theta, B,\rho,\ell, c)>0$ is some constant.
\end{theorem}

   For independent sampling, $\Var[N^{-1}  L_{\mathcal S}(f)]$ is of order $O(m^{-1})$. In contrast, sampling with DPPs often results in a variance of order $ O(m^{-(1+\delta)})$ for some $\delta > 0$ (e.g., see \cite{OPE-AOAP, OPE-NIPS}). 
   Therefore, the range of $\varepsilon$ for which we can apply our concentration result is $ O(m^{-\delta})$. 
    Substituting $\varepsilon = m^{-\alpha} (\alpha \ge \delta)$ into Theorem \ref{t:core-union}  yields the upper bounds $2\exp(6D - C' m^{1+\delta - 2\alpha})$ and $2\exp(CD + \alpha D \log m - C' m^{1+\delta - 2\alpha}) $, where $C$ and $C'$ are positive constants independent of $m$ and $N$. Both of them converge to $0$ as $m\rightarrow \infty$, provided $\alpha < (1+\delta)/2$. Hence, the accuracy rate $\varepsilon$ can be as small as $m^{-1/2 -\delta'/2}$, for any $0<\delta' <\delta$, which clearly outperforms the best accuracy rate of $m^{-1/2}$ for i.i.d. sampling.

The assumptions in Theorem \ref{t:core-union} are fairly moderate and reasonable. For i.i.d. sample $\mathcal S$ with expected size $m$, the probability $\mathbb P(x \in \mathcal S) = m/N$  applies to each $x\in \mathcal X$. For a DPP sample $\mathcal S$ with kernel $\bK$, we have $\mathbb P(x_i\in \mathcal S) = \bK_{ii}$. 
Therefore, assuming $\bK_{ii} \ge \rho \cdot m/N$ for each $i$ and some $\rho>0$ ensures that every point in the dataset $\mathcal X$ has a fair chance of being selected. 
This also helps prevent the estimated loss $L_{\mathcal S}(f) = \sum_{x_i\in \mathcal S} f(x_i)/\bK_{ii}$ from blowing up to infinity. In the second scenario, the assumption $|\mathcal S|\le B \cdot m$ a.s. for some $B>0$ is not strictly required but is included to simplify the presentation of the results. Indeed, for a DPP $\mathcal S$ with Hermitian kernel, it is known that $|\mathcal S|$ is equal in distribution to the sum of independent Bernoulli variables, which consequently implies that $|\mathcal S|$ highly concentrates around its mean $m$. This property allows us to relax the assumption on the size of $|\mathcal S|$ by applying a conditioning argument.  However, the  assumption $|\mathcal S| \le B \cdot m$ a.s. holds for most kernels of interest; DPPs with projection kernels being typical and significant examples. In machine learning terms, this means that the coresets are not much bigger than their expected size $m$; whereas in practice, sampling methods often produce coresets of a fixed size, as observed with projection DPPs. We refer the readers to \cite{NIPS2024} for more detailed discussion.

%% file: networks.tex
\subsection{Background}
The use of point processes to model stochastic network configurations has a fairly long history. Indeed, a key aspect of wireless communication networks is their spatial nature. These networks consist of numerous nodes distributed across space, typically in two- or three-dimensional Euclidean space. Each node of the network, modelling a wireless base station, broadcasts signals that interfere with one another, making it crucial to understand the distribution of signal strength across different locations, accounting for these interferences. One main objective is to design network layouts that maximize signal strength for the majority of locations.

This area is also intertwined closely with the stochastic geometry of point processes. The stochastic geometry of negatively dependent point processes has attracted considerable attention in recent years in the applied probability community, including but not limited to the problems of continuum percolation or the Gilbert disk model \cite{yogesh, Baccelli-Blasz, Blaszczyszyn_Haenggi_Keeler_Mukherjee_2018, shirai-miyoshi, ghosh_aop}.

\subsection{Stochastic Networks and hyperuniformity}
The classical and most popularly studied model for stochastic networks is based on the Poisson point process \cite{Baccelli-Blasz,Blaszczyszyn_Haenggi_Keeler_Mukherjee_2018,Haenggi}. However, Poissonian networks exhibit similar issues to independent random points, thereby limiting their performance. 

In view of this, dependent point fields have been investigated in this area; most notably the so-called \textit{Ginibre network} which is based on the Ginibre ensemble -- a DPP that is intimately related to random matrices and 2D Coulomb gases. In the celebrated work \cite{shirai-miyoshi}, the authors proposed a stochastic network based on the eigenvalues from the Ginibre ensemble, and demonstrated rigorously that such a network provides provably improved performance over standard Poissonian networks. To this end, they observed that for quantities that depend solely on the absolute values of the Ginibre points, such as the coverage probability or link success probability, we can take advantage of the fact that the process of absolute values can be equivalently described by independent gamma random variables.

However, Ginibre networks have difficulties vis-a-vis tractability and robustness issues. For instance, from an application point of view, the network should be easy to simulate, so that large scale statistical behaviour can be easily understood from empirical studies. It is reasonable to believe that many of the salutary properties of a Ginibre network are consequences of not the specific, delicate structure of the Ginibre model itself, but of the more general phenomenon of \textit{hyperuniformity}. 

Hyperuniformity refers to the suppression of fluctuations in particle numbers in a spatial point process. In Poissonian systems, the variance in the number of points within a large spatial domain increases proportionally to the volume (a phenomenon described in physics as extensive fluctuations). However, in hyperuniform systems, these fluctuations are of a smaller order relative to the volume (sub-extensive), often scaling only with the surface area of the domain. Such behaviour manifests itself in a wide array of natural systems  \cite{hough2009zeros, GhoshLebowitz, TORQUATO, torquato2, ghosh-lebowitz-cmp}, and allows for elegant and effective interpretations in terms of their \textit{power spectra}  \cite{PhysRevE,baake, ghoshlebowitz-cmp2}. We refer the readers to \cite{GhoshLebowitz, coste} for a detailed discussion on hyperuniformity, and to \cite{lacieze_hyperuniform} for statistical tests of hyperuniformity for point processes.

\subsection{Stochastic networks based on disordered lattices}
In \cite{ghosh-shirai}, the authors investigate disordered lattice networks, aiming to find a best-of-both-worlds compromise between Poissonian and Ginibre networks. As the name indicates, these networks are based on point processes that are i.i.d. random perturbations of Euclidean lattices. It is known that under very general conditions -- as soon as the tail of the perturbing random variable decays faster that $\|x\|^{-d}$ for a $d$-dimensional lattice -- the resulting disordered lattice constitutes a hyperuniform process \cite{GhoshLebowitz, Ghosh2016NumberRI}. 

By modulating the statistical law of the perturbations to vary within a parametric family of distributions (e.g. an exponential family), disordered lattices allow for the developments of a parametric statistical theory in the context of stochastic networks. In the particular setting of Gaussian perturbations, modulating the scale parameter of the Gaussian provides a smooth interpolation between stochastic networks of varying degrees of disorder, with a rigid lattice at one end and a Poisson network at the other. 

A typical instance of the network statistics we want to investigate is the so-called \textit{Signal to Interference plus Noise Ratio} (abbrv. SINR) at a typical point (let's say the origin). Due to the randomness of the network, this is a random variable, and therefore it is of interest to investigate its tails, also known as coverage probabilities:
\begin{equation} \label{eq:infiniteprod_d}
p_c(\theta, \beta) =  \P(\mathrm{SINR} > \theta)
= \E\left[\prod_{j \not= B} \left(1 +\theta
 \frac{|X_B|^{d\beta}}{|X_j|^{d\beta}}  \right)^{-1}
\right],
\end{equation} 
where $X_B$ is the location of the node of the network nearest to the origin, $d$ is the ambient dimension, and $\beta>0$ is a fixed modelling parameter related to the decay of the so-called \textit{path loss function} \cite{Baccelli-Blasz, shirai-miyoshi}.

Focusing on Gaussian disorder for the perturbed lattice networks, the authors in \cite{ghosh-shirai} demonstrate that for an appropriate choice of the disorder strength, the coverage probabilities roughly match those of a Ginibre random network, thereby achieving the \textit{best of both worlds} phenomenon aspired for. 

Interestingly, at the same disorder strength, the perturbed lattice point process appears to be the \textit{closest} to the Ginibre random point field. In order to achieve this comparison, \cite{ghosh-shirai} introduces an approach towards robust measurement of the ``distance'' between point processes by a quantitative comparison of their \textit{persistence homology}, a technique that we expect to see increasing usefulness in coming years. 

In addition to elaborate theoretical analysis demonstrating power law asymptotics for the coverage probability of disordered lattice networks and of interpolation properties discussed above, \cite{ghosh-shirai} obtains stochastic approximations to the SINR variable in various limiting regimes of interest. An interesting finding in this vein is the emergence of the Epstein Zeta Function of the underlying lattice as a key determinant of network statistics. This leads to the optimal choices of lattices as the minimizers of this lattice energy function, which is given by the triangular lattice in 2D and face centered cubic (FCC) lattice in 3D.

%% file: spectrogram.tex
In time-frequency analysis, Short Time Fourier Transforms (abbrv. STFTs) are foundational objects \cite{TFA_Grochenig, TFA_Flandrin, TFA_Cohen}. 
Given a signal $f\in L^2(\R)$ and a window function $\phi \in L^2(\R)$, the STFT $\cF_\phi f(\cdot,\cdot)$ of $f$ with respect to $\phi$ is defined as the inner product between f and shifted $\phi$ (in the sense of time and frequency). More precisely,
\begin{equation}
    \cF_\phi f(u,v) := \int_\R f(x) \overline{\phi (x-u)} e^{-2\pi i x v} \d x = \langle f , M_v T_u \phi \rangle,
\end{equation}
where $T_u g (x) := g(x- u)$ and $M_v g(x) := e^{2\pi i v x} g(x) $ and $\langle f,g \rangle := \int_\R f(x)\overline{g(x)}\d x$.
\begin{remark}
    It is convenient to work with complex notation sometimes. By setting $z= u + iv$, we will regard $\cF_\phi f$ as function of $z\in \C$.
\end{remark}

Popular choices for window functions $\phi$ include Gaussians
\[\phi(x) \propto e^{-x^2/2\sigma^2}, \quad \sigma >0.\]
We will particularly choose  $\phi (x) = 2^{1/4}e^{-\pi x^2}$ as in \cite{BARDENET_Spec,GhoshGAF}, 
to make use of the \emph{Bargmann transform}, which will be introduced later. In this setting, the STFT is also known as the \emph{Gabor transform} \cite{TFA_Grochenig}.

Given the STFT of a signal, the function $u,v\mapsto|\cF_\phi f(u,v)|^2$ is  fundamentally important, called the \emph{spectrogram} of the signal $f$ with respect to the window function $\phi$.
We will be particularly interested in the \emph{level sets} of the spectrogram of $f$
\[\mathcal L (a) := \{(u,v) \in \R^2: |\cF_\phi f(u,v)| \ge a \}, \quad a\ge 0.\]

Spectrogram analysis is highly effective in many applications \cite{TFA_Grochenig, Spectrogram_Flandrin}. Classically, the maxima of the spectrogram have been extensively studied. This is related to the understanding that these capture greater energy from the spectrogram and, therefore, provide more information about the signal  \cite{TFA_Grochenig, TFA_Flandrin, TFA_Cohen, Flandrin_2018}. However, in recent years, a different approach to spectrogram analysis has been investigated, focusing on the zeros of the spectrogram instead of the maxima. This new line of investigation originated from a seminal paper by Flandrin \cite{Spectrogram_Flandrin}. 

In relation to this direction, Bardenet et al. \cite{BARDENET_Spec} observed that the zero set of the spectrogram of the Gaussian white noise has the same statistical distribution as the zero set of the planar Gaussian Analytic Function (GAF). In \cite{BARDENET_GAF}, Bardenet and Hardy further explored the connection between several time-frequency transforms of Gaussian white noises and different models of GAFs (planar, elliptic, hyperbolic). In particular, the STFT with respect to a Gaussian window corresponds to the planar GAF, and the analytic wavelet transform \cite{Daubechies_1988} is related to the hyperbolic GAF. Thus, spectrogram analysis is linked to the theory of Gaussian Analytic Functions, a rich and elegant topic of independent mathematical interest \cite{hough2009zeros, PeresVirag, sodin}. 

\subsection{Zeros of Gaussian Analytic Functions}
To be self-contained and for the sake of the readers, we will briefly review the Gaussian analytic functions and their zeros.
\begin{definition}[Gaussian analytic functions, \cite{hough2009zeros}]
    Let $D \subset \mathbb C$ be a domain. A Gaussian analytic function (GAF) on $D$ is a random variable $f$ taking values in the space of analytic functions on $D$ (endowed with the topology of uniform convergence on compact subsets), such that for every $z_1,\ldots,z_n \in D$ and  $n \ge 1$, the random vector $(f(z_1),\ldots,f(z_n))$ has a mean zero complex Gaussian distribution.
\end{definition}

Given a GAF $f$ on a domain $D$, it is natural to consider the covariance kernel of $f$, defined by
\[ K(z,w) := \mathbb E[f(z) \overline{f(w)}], \quad z,w \in \mathbb C.\]
Due to the Gaussian nature and the continuity of $f$, the covariance kernel $K$ completely determines the distribution of $f$.

A standard way to construct a GAF is as follows.

\begin{proposition}
    Let $D$ be a domain in $\mathbb C$ and $\Psi_n, n \ge 1$ be holomorphic functions on $D$ such that $\sum_{n} |\Psi_n(z)|^2$ converges uniformly on compact subsets of $D$. Let $\xi_n, n \ge 1$ be i.i.d. complex Gaussian random variables. Then, almost surely, the random series
    \[ f(z) := \sum_{n=1}^\infty \xi_n \Psi_n(z)\]
    converges uniformly on compact subsets of $D$, and hence defines a Gaussian analytic function $f$ with covariance kernel
    \[ K(z,w) = \sum_{n\ge 1} \Psi_n(z) \overline{\Psi_n(w)}.\]
\end{proposition}

Our main interest is in the zeros of GAFs. There are three particular models of GAFs, for which the corresponding zero set interacts nicely with the geometry of the ambient domain. In what follows, $\xi_n, n\ge 0$ are i.i.d. $\mathcal N_{\mathbb C}(0,1)$-random variables.

\begin{enumerate}
    \item(Planar GAF) Let $D= \mathbb C$, we define
    \[ f_L(z) := \sum_{n=0}^\infty \xi_n \sqrt{\frac{L^n}{n!}}z^n, \quad L>0.\]
    For each $L>0$, $f_L(z)$ is a GAF on $\mathbb C$ with covariance kernel $\exp(Lz\bar w)$.
    \item(Spherical GAF)
    Let $D= \mathbb C \mathbb P^1 = \mathbb C \cup \{\infty\}$, we define 
    \[f_L(z) := \sum_{n=0}^L \xi_n \sqrt{\frac{L(L-1)\ldots(L-n+1)}{n!}}z^n, \quad L \in \mathbb N=\{1,2,3,\ldots\}.\]
    For each $L\in \mathbb N$, $f_L(z)$ is a GAF on $\mathbb C$ with covariance kernel $(1+z\bar w)^L$, and it has a pole at $\infty$.
    \item(Hyperbolic GAF) Let $D= \mathbb D = \{|z|<1\}$, we define 
    \[f_L(z) := \sum_{n=0}^L \xi_n \sqrt{\frac{L(L+1)\ldots(L+n-1)}{n!}}z^n, \quad L >0.\]
    For each $L>0$, $f_L(z)$ is a GAF on $\mathbb D$ with covariance kernel $(1-z\bar w)^{-L}$. 
\end{enumerate}

\begin{figure}[h]
    \centering
    \includegraphics[width=0.80\linewidth]{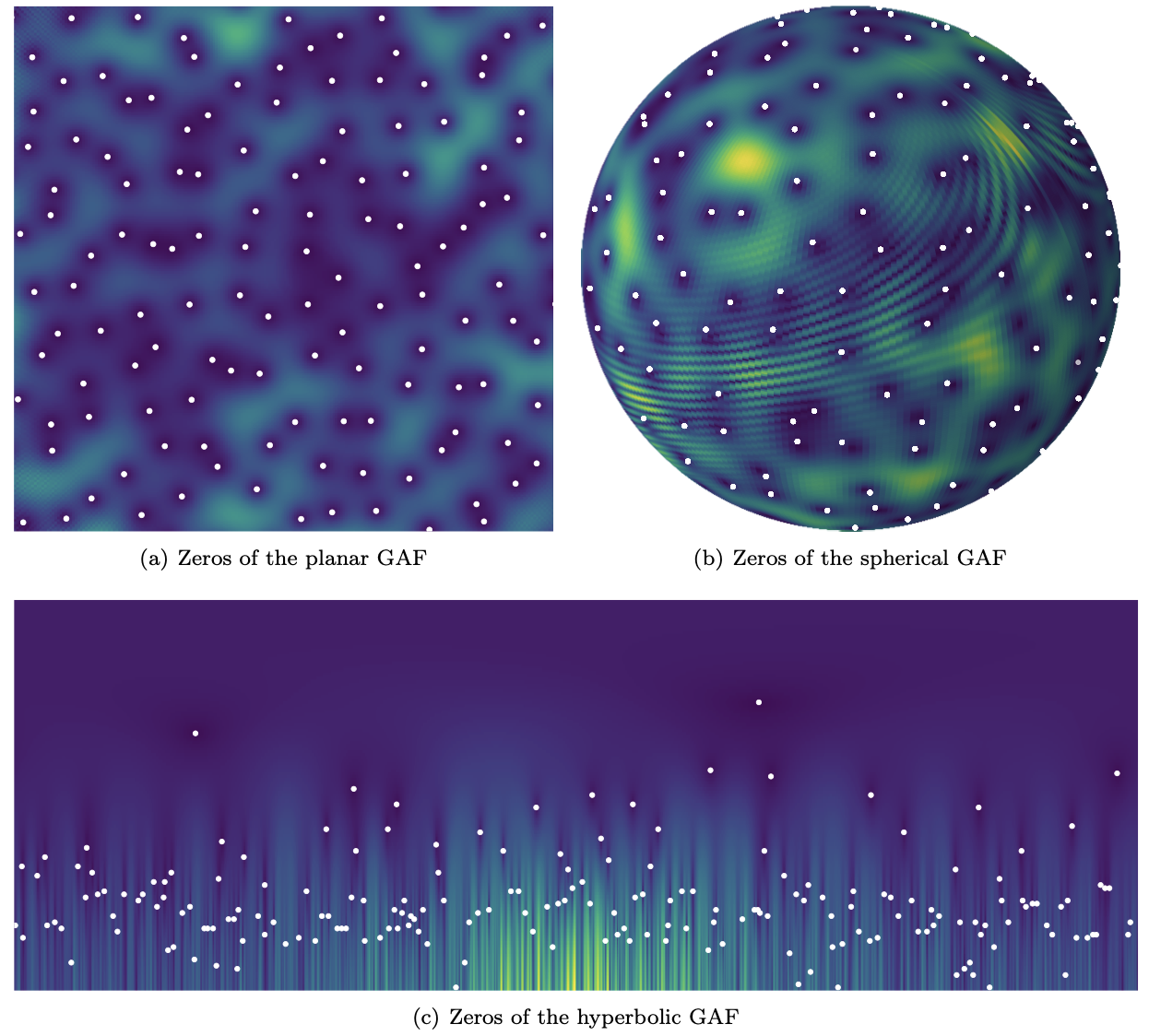}
    \caption{Zeros of 3 canonical models of GAFs (cf. \cite{BARDENET_GAF})}
    \label{fig:zeros}
\end{figure}

These 3 models are characterized by the property that the zero sets are invariant (in distribution) and ergodic under the action of the group of isomorphisms of the ambient domains (see \cite{hough2009zeros}, Chapter 2). 

The behavior of the zeros of GAFs is quite similar to that of determinantal point processes (DPPs), as both exhibit the so-called \emph{hyperuniformity} in the spatial distribution of points, which is of independent interest in statistical physics \cite{Ghosh2016NumberRI, GhoshLebowitz, PhysRevE}. 
In fact, for the hyperbolic GAF with $L=1$, it has been shown by Peres and Virag \cite{PeresVirag} that the zeros of $f_1$ is a DPP on $\mathbb D$ with the Bergman kernel $K(z,w) = \pi^{-1} (1-z\bar w)^{-2}$ and the background measure is the Lebesgue measure on $\mathbb D$. The zeros of GAFs also exhibit repulsiveness, but at short ranges. In particular, the probability of having 2 zeros in a same small disk of radius $\varepsilon$ is of order $\mathcal{O}(\varepsilon^6)$ (see \cite{hough2009zeros}, Chapter 3). In contrast, from \eqref{eq:Poisson_repel} with $d=2$, that probability for Poisson point processes is of order $\varepsilon^4$.
For more details about GAFs, we refer the readers to \cite{hough2009zeros, sodin}.

\subsection{Bargmann transform and Gaussian white noise}
\subsubsection{Bargmann transform and Hermite functions}
\begin{definition}[Bargmann transform]
Given $f\in L^2(\R)$, the Bargmann transform of $f$ is a function on $\C$ defined by
\[ \mathcal Bf(z) = 2^{1/4} \int_\R f(x) \exp \Big ( 2\pi x z - \pi x^2 - \frac{\pi}{2} z^2 \Big ) \d x, \quad z\in \C.\]
\end{definition}

The connection of the STFT (with respect to Gaussian window functions) to the Bargmann transform is given by the following proposition:  

\begin{proposition}
For $f\in L^2(\R)$ and $z= u+iv$, we have 
\[ \cF_\phi f (z) = \exp \Big ( -\pi i u v - \frac{\pi}{2} |z|^2 \Big ) \mathcal B f(\bar z), \]
where $\phi (x) := 2^{1/4}e^{-\pi x^2}$.
\end{proposition}

The Hermite functions $h_k(x), k=0,1,2,\ldots$ are defined as 
\begin{equation} \label{eq:hermite}
 h_k(x) := \frac{(-1)^k}{2^{2k-1/2} \pi^k k!} e^{\pi x^2} \frac{ \d^k}{\d x^k} \Big (e^{-2\pi x^2} \Big ), \quad k \in \N.
\end{equation}
It is known that $\{h_k\}_k$ is an orthonormal basis of $L^2(\R)$ and they behave nicely under Bargmann transform:

\begin{proposition}[\cite{BARDENET_Spec,GhoshGAF}] \label{p:Bargmann}
We have
\begin{enumerate}
\item[(i)] $ \mathcal Bh_k(z) = \frac{\pi^{k/2}}{\sqrt{k!}} z^k , \quad k\in \N$.
\item[(ii)] $\cF_\phi h_k (z) = \exp  (-\pi i uv - \frac{\pi}{2} |z|^2  ) \sqrt{\frac{\pi^k}{k!}} \bar{z}^k$, where $z= u+iv$ and $k\in \N$.
\end{enumerate}
\end{proposition}

\subsubsection{Gaussian white noise}

To define the white noise, we recall the notions of Schwartz space and tempered distributions. Let $C^\infty(\R,\C)$ be the space of smooth functions from $\R$ to $\C$, the Schwartz space $\cS(\R)$ is defined as
\[ \cS(\R) = \Big \{ f \in C^\infty(\R, \C) : \sup_{x\in \R} |x^n f^{(m)}(x)| < \infty, \: \forall m,n \in \N \Big \}.\]
The Schwartz space particularly contains Gaussian density functions.

The space of tempered distributions on $\R$ is the continuous dual of $\cS(\R)$, for which we will denote by $\cS'(\R)$. For $\psi \in \cS'(\R)$ and $\phi\in \cS(\R)$, we denote the action of $\psi$ on $\phi$ by $\langle \psi , \phi \rangle $.

\begin{definition}[White noise]
Let $\mathcal B(\cS'(\R))$ be the Borel $\sigma$-field on $\cS'(\R)$. The white noise $\xi$ is the random element of $(\cS'(\R), \mathcal B(\cS'(\R)))$ satisfying
\[ \E[e^{i\langle \xi, g \rangle } ] = e^{-\frac{1}{2} \|g\|^2_{L^2(\R)}}, \quad \forall g\in \cS(\R). \]
In other words, $\langle \xi ,g \rangle$ is a $\mathcal N(0,\|g\|_{L^2(\R)}^2)$-random variable for all $g\in \cS(\R)$. We will denote by $\mu_1$ the law of $\xi$.
\end{definition}

\subsection{The short-time Fourier transforms of white noise}

We note that the notion of STFT extends naturally for $f\in \cS'(\R)$ and $\phi \in \cS(\R)$ as
\[ \cF_\phi f(u,v ) := \langle f , M_v T_u \phi \rangle , \quad u,v\in \R.\]
Using Bargmann transform, the STFT of the white noise $\xi$ is then given by the following theorem (see \cite{BARDENET_Spec, GhoshGAF}).
\begin{theorem} \label{p:whitenoise}
For $z= u+iv \in \C$ and $\phi (x) = 2^{1/4}e^{-\pi x^2}$, we have
\begin{equation} \label{eq:stft_whitenoise}
\cF_\phi \xi (z) = \sqrt{\pi} \exp\Big (i\pi uv - \frac{\pi}{2}|z|^2 \Big ) \sum_{k=0}^\infty \langle \xi, h_k \rangle \sqrt{\frac{\pi^k}{k!}} z^k,
\end{equation}
where the convergence in the RHS is in $L^2(\mu_1)$. In particular, the random series
\[ \sum_{k=0}^\infty \langle \xi, h_k \rangle  \sqrt{\frac{\pi^k}{k!}} z^k\]
defines a random entire function $\mu_1$-a.s.
\end{theorem}

Thus, the STFT of white noise is a Gaussian analytic function on $\mathbb C$. The covariance kernel of $\{\cF_\phi \xi (z)\}_{z\in \C}$ is given by
\begin{equation}
    K(z,w) := \Cov[\cF_\phi \xi(z) , \cF_\phi \xi(w)] = \pi \exp\Big ( i\pi(u_1v_1 - u_2v_2) - \frac{\pi}{2}(|z|^2 + |w|^2) + \pi z\bar w \Big ),
\end{equation}
where $z= u_1 + i v_1, w = u_2 + i v_2$. In particular,
    $K(z,z) = \Var[\cF_\phi \xi (z)] = \pi$.

Theorem \ref{p:whitenoise} particularly implies that the zeros of the STFT of the white noise are the same as those of 
\[ \sum_{k=0}^\infty \langle \xi, h_k \rangle  \sqrt{\frac{\pi^k}{k!}} z^k.\]
Since $\langle \xi, h_k \rangle$ are i.i.d. real Gaussians, this is also called the \emph{symmetric planar Gaussian analytic function}. However, the zero set in this case is not stationary, which poses an obstacle in using techniques from spatial statistics. Motivated by early results by Prosen, the authors of \cite{BARDENET_Spec} argue that the zero set of the spectrogram of real Gaussian white noise could be effectively approximated by the zero set of the spectrogram of \emph{complex} Gaussian white noise. It turns out that the latter is precisely the zero set of the planar GAF, which has been shown to be invariant under isomorphisms of the complex plane.

\subsection{Applications in signal analysis}
The connection to the zeros of GAFs provides a useful approach for studying the stochastic geometry of spectrogram level sets of a signal. In \cite{GhoshGAF}, Ghosh et al. provided a theoretical guarantee for using GAFs in the problem of signal detection. To elaborate,
let $f$ be a noisy signal generated as
\begin{equation}\label{eq:signal}
     f = \lambda h_k + \xi
\end{equation}
where $\lambda \in \R$ is a real parameter, $h_k$ is an Hermite function defined as \eqref{eq:hermite} and $\xi$ is the white noise. We are interested in the spectrogram level sets of $f$ and their restrictions
\begin{eqnarray} \label{eq:spectrogram}
    \mathcal L(a) &:=& \{z\in \C: |\cF_\phi f(z)| \ge a \}, \quad a\ge 0 \\
    \mathcal L_L(a) &:=& \mathcal L(a) \cap \B_L, \quad \text{where } \B_L:=\{z\in \C: |\re(z)|\le L, |\im(z)|\le L\}.
\end{eqnarray}

Utilising Equation \eqref{eq:stft_whitenoise} and techniques from maxima of Gaussian random fields, Ghosh et al. in \cite{GhoshGAF} showed the following result for the spectrogram level set of the white noise
\begin{theorem} \label{t:sup_whitenoise}
    For $L\ge \pi$, we have that for any $\tau > 0$
    \[ \P \Big ( \sup_{z\in \B_L}|\cF_\phi \xi (z)| \le \sqrt 2 (14K +\tau) \sqrt{\log L}\Big ) \ge 1 - 4 \exp \Big (- \frac{\tau^2}{2 \pi} \log L \Big )\]
    where $K>0$ is a constant, and
    \[ \B_L:= \{z\in \C: |\re(z)|\le L, |\im(z)|\le L\}.\]
\end{theorem}

We note that  $\cF_\phi$ is linear and $f$ is a mixture of an Hermite function and the white noise. Thus, to study spectrogram level sets of $f$, one can utilise Theorem \ref{t:sup_whitenoise} and the fact that
\begin{equation}
    \sup_{z\in \C} |\cF_\phi h_k(z)| = \prod_{n=1}^k\sqrt{\frac{k}{en}},
\end{equation}
and the supremum is attained at $|z|=\sqrt{k/\pi}$ (see \cite{GhoshGAF}). 
Combining the ingredients above, we have the following theorem.
\begin{theorem} \label{t:noisy-signal}
    Let $f$ be a noisy signal generated as \eqref{eq:signal}. Let $L\ge \max \{ \sqrt{k/\pi}, \pi\}$ and 
    \[ |\lambda | \ge \frac{5\sqrt{2}(14K + \tau) \sqrt{\log L}}{\prod_{n=1}^k \sqrt{k/(en)}}\]
    where $K>0$ is the constant in  Theorem \ref{t:sup_whitenoise} and $\tau>0$ is a parameter. Then with probability at least $1-4\exp(-\frac{\tau^2}{2\pi} \log L )$, we have
    \[ \varnothing  \neq \mathcal L_L\Big (3\sqrt{2}(14K+\tau)\log L \Big ) \subset \Big \{z\in \C: \frac{|\cF_\phi h_k(z)|}{\prod_{n=1}^k \sqrt{k/(en)}} > \alpha\Big \}\]
    where
    \[\alpha := \frac{\sqrt{2}(14K + \tau) \sqrt{\log L}}{|\lambda|\prod_{n=1}^k \sqrt{k/(en)}} \in (0,1/5]. \]
\end{theorem}

Theorem \ref{t:noisy-signal} provides useful technical tools in signal detection. In particular, the authors in \cite{GhoshGAF} present two applications: one in hypothesis testing and one in signal estimation. For the first application, they design an effective hypothesis test to distinguish between pure white noise and the presence of a fundamental signal. They also provide theoretical guarantees for its effectiveness using non-asymptotic bounds on the power of the test. For the second application, they demonstrate an estimation procedure for identifying a fundamental signal, if present, and provide error bounds on the accuracy of this estimation based on the sample size. For more details, we refer readers to Section 4 of \cite{GhoshGAF}.

%% file: columnsubsetselection.tex
This section is based on \cite{belhadji_determinantal_2020}.
It is well known that the best (e.g. in Frobenius norm) approximation of a given matrix $X \in \mathbb{R}^{N \times  d}$ by a matrix of smaller rank $k \ll d$ is obtained by principal component analysis (PCA). The main drawback from this is that one loses the explainability of the features (columns), which can be particularly troublesome in some areas. Column subset selection aims at resolving this particular issue: for the matrix $X$ containing $d$ features about $N$ individuals, feature selection consists in selecting $k \ll d$ columns that are the most relevant. Finding the k-columns which contributes the most to the matrix is a complex task, and randomized methods have proven to be particularly efficient in this regard.

To fix the notations, for a matrix $A = (A_{ij})$, and subsets $I$, $J$, let $A_{IJ} = (A_{ij})_{i \in I, j \in J}$. 

Assume $\operatorname{rk}(X) = r$, fix also the singular value decomposition (SVD) of $X$ as $U \Sigma V^T$, where the singular values of $X$, $\sigma_i$, are ordered in decreasing order, $V \in \mathbb{R}^{d  \times r}$, $U \in \mathbb{R}^{N  \times r}$ are orthogonal and $\Sigma \in \mathbb{R}^{r  \times r}$ is diagonal. Let $V_k = V_{:[k]} = (V_{ij})_{i \in [d], j \in [k]}$ for $k\leq d$.

The quality of the approximation is seen through the ability of the span of the submatrix $X_{:S} = (X_{ij})_{i \in [n], j \in S}$ to approximate the columns of $X$, hence the criterion:

\[ \mathcal{L}_{k} : S  \subset  \{1, \dots, d\} \rightarrow  \underset{B \in \mathbb{R}^{\lvert S \rvert \times d} \text{ , } rk(B) \leq k}{\operatorname{min}} \lVert X - X_{:S} B \rVert_{\operatorname{Fr}}^2 \]

Call $\Pi_{S,k}X = X_{:S} X_{:S}^+ X $ (which is of rank at most $k$), where $X_{:S}^+$ is the Moore-Penrose pseudo-inverse of $X_{:S}$, the above criterion then reads, $\mathcal{L}_{k}(S) = \lVert X - \Pi_{S,k}X  \rVert_{\operatorname{Fr}}^2 $. 

Denote by $\Pi_k X$ the best rank $k$ approximation of $X$ (in the Frobenius sense), which can be obtained through classical PCA. Finally define the $k$-leverage scores as the norms of the columns of $V_k^T$, for $1\leq j \leq d$, $l_j^k = \sum \limits_{i = 1}^k V_{ij}^2$.

Solving the problem of finding a good candidate $S$ is a challenge. Deterministic algorithms exist such as rank revealing QR (RRQR) decomposition, see \cite{boutsidis_improved_nodate}, which have their own caveats. Here we choose to focus on non-deterministic methods.

\subsection{Non-repulsive algorithms}

\subsubsection{Length squared importance sampling}

The first idea to randomly select columns is to give them importance scores as their norm. This algorithm was introduced by Drineas et al. in \cite{drineas_clustering_2004}. More precisely, the idea is to sample $s$ indices, from a multinomial distribution on $[d]$, where the probability to sample $i$ is  $p_i \propto \lVert X_{:,i} \rVert^2$.

Theorem 3 in \cite{drineas_clustering_2004} states that:

 \[ \mathbb{P}( \lVert X - \Pi_{S,k}X  \rVert_{\operatorname{Fr}}^2 \leq \lVert X - \Pi_{k}X  \rVert_{\operatorname{Fr}}^2 + 2(1 + \sqrt{8 \log(\dfrac{2}{\delta} )}) \sqrt{\dfrac{k}{s}} \lVert X \rVert_{\operatorname{Fr}}^2) \geq 1 + \delta   \]

Notice first that the bound is additive here. Besides, this method leads to sample $s \gg k$ columns to achieve a good rank $k$ approximation.

\subsubsection{Leverage score importance sampling}

Another importance score that can be considered is the leverage score. This method was introduced by Drineas et al. in \cite{drineas_relative-error_2008}. The $k$-leverage scores distribution is the multinomial distribution on $[d]$ with parameters $p_i = l_i^k/k$. 
It is easy to see that the leverage scores sum to $k$ so that this definition gives a proper probability distribution.
Besides, for $\mathcal{P}_k = \operatorname{Span}(V_k)$, when $(e_i)_{i \in [d]}$ is the canonical basis of $\mathbb{R}^d$, and $\theta_i$ is the angle between $\mathcal{P}_k$ and $e_i$, some linear algebra manipulations show that $l_i^k = \cos^2(\theta_i) $.
So selecting the columns according to the leverage score sampling is in fact choosing the coordinates $e_i$ that are the most aligned with $\mathcal{P}_k$. PCA selects exactly $\mathcal{P}_k$,  hence, selecting a subspace which is almost aligned with $\mathcal{P}_k$ should yield a good approximation. 

Theorem 3 in \cite{drineas_relative-error_2008} states that, for $\epsilon>0$, $\delta > 0$, if the number of sampled columns $s$ satisfies $s \geq 4000(k^2/\epsilon^2)\log(1/\delta)$, then under the $k$-leverage score sampling it holds:

    \[ \mathbb{P}( \lVert X - \Pi_{S,k}X \rVert_{\operatorname{Fr}} \leq (1+ \epsilon) \lVert X - \Pi_kX \rVert_{\operatorname{Fr}}) \geq 1 - \delta \]

Finally, the lower bound on $s$ is sharp.
Likewise, the main drawback from this method is that to achieve a good rank $k$ approximation of $X$, one has to sample $s \gg k$ columns.

\subsection{Repulsive algorithms} 

\subsubsection{Volume sampling}

One way to introduce correlation between the columns of $X$ is through volume sampling. The idea is here to use the kernel $X^TX$ and sample a $k$-DPP from it. This method will be called volume sampling (VS) in the following. Notice that the diagonal terms of the kernel are exactly the squared lenghts of the columns of $X$.
Volume sampling was first introduced by Deshpande et al. in \cite{deshpande_matrix_2006} and they established the following result:

\begin{proposition} \label{prop:VS} [\cite{deshpande_matrix_2006} Theorem 1.3]
    Assume $\mathcal{S}$ is sampled according to the volume sampling method. Then:
 
\begin{equation}
 \mathbb{E}_{\operatorname{VS}} [\lVert X - \Pi_{S,k}X \rVert_{\operatorname{Fr}}^2] \leq (k+1) \lVert X - \Pi_k X \rVert_{\operatorname{Fr}}^2 
\end{equation}

Besides, this bound is sharp in the worst case scenario.
\end{proposition}

The proof scheme for this result revolves around algebraic manipulations of the determinant. More precisely, expending the expectation gives: 

\[ \mathbb{E}_{\operatorname{VS}} [\lVert X - \Pi_{S,k}X \rVert_{\operatorname{Fr}}^2] = \dfrac{1}{\sum \limits_{S, \lvert S \rvert = k } \operatorname{det}(X_{:S}^TX_{:S})} \sum \limits_{S, \lvert S \rvert = k } \operatorname{det}(X_{:S}^TX_{:S} ) \lVert X - \Pi_{S,k}X \rVert_{\operatorname{Fr}}^2 \]

Now, for $T$ such that $\lvert T \rvert = k+1$, $T = \{i_1, \dots, i_{k+1}\}$, if $S = \{i_1, \dots, i_k\}$, the "base times height" formula states that: 

\[ \operatorname{det}(X_{:T}^TX_{:T}) = \dfrac{\operatorname{det}(X_{:T}^TX_{:T}) \operatorname{d}(X_{:,i_{k+1}}, \operatorname{span}\{X_{:,i_j} \text{ , } j \leq k\})^2}{(k+1)^2}  \]

Where $\operatorname{d}$ is the distance. Taking the sum of this when $T$ ranges over the parts of size $k+1$, and noticing that $\sum \limits_{i = 1}^d \operatorname{d}(X_{:,i}, \operatorname{span}\{X_{:,j} \text{ , } j \in S^2\})^2 =  \lVert X - \Pi_{S,k}X \rVert_{\operatorname{Fr}}^2$ gives:

\[ \sum \limits_{T, \lvert T \rvert = k +1 } \operatorname{det}(X_{:T}^TX_{:T}) = \dfrac{1}{(k+1)^3} \sum \limits_{S, \lvert S \rvert = k } \operatorname{det}(X_{:S}^TX_{:S} ) \lVert X - \Pi_{S,k}X \rVert_{\operatorname{Fr}}^2  \]

To conclude, the authors use a technical lemma on symmetric polynomials:

\[\sum \limits_{S, \lvert S \rvert = k } \operatorname{det}(X_{:S}^TX_{:S} ) = \dfrac{1}{(k!)^2} \sum \limits_{1 \leq t_1 < \dots < t_k \leq r } \sigma_{t_1}^2 \dots \sigma_{t_k}^2\]

This technical lemma allows them to write:

\[\sum \limits_{1 \leq t_1 < \dots < t_{k+1} \leq r } \sigma_{t_1}^2 \dots \sigma_{t_{k+1}}^2 \leq \sum \limits_{1 \leq t_1 < \dots < t_k \leq r } \sigma_{t_1}^2 \dots \sigma_{t_k}^2 \sum\limits_{j=k+1}^r \sigma_j^2 = (k!)^2 \sum \limits_{S, \lvert S \rvert = k } \operatorname{det}(X_{:S}^TX_{:S} ) \lVert X - \Pi_{k}X \rVert_{\operatorname{Fr}}^2  \]

Finally, the theorem follows from applying the technical lemma and then the aforementioned inequality on:

\[\mathbb{E}_{\operatorname{VS}} [\lVert X - \Pi_{S,k}X \rVert_{\operatorname{Fr}}^2] = \dfrac{(k+1)^3}{\sum \limits_{S, \lvert S \rvert = k } \operatorname{det}(X_{:S}^TX_{:S})} \sum \limits_{T, \lvert T \rvert = k +1 } \operatorname{det}X_{:T}^TX_{:T}  \]

Notice that this method samples exactly $k$ columns, unlike the previous i.i.d. schemes discussed and still enjoys strong theoretical properties.

Recently, \cite{derezinski_improved_2021} have worked on these bounds and they proved the following result. Notice that the bound is sharp in the worst case scenario, so they play on the decay of the singular values to achieve a better bound in some cases.

\begin{proposition} [\cite{derezinski_improved_2021} Theorem 1]
    For $s< r$, the stable rank of $X$, $sr_s(X)$ is $\sigma_{s+1}^{-2} \sum \limits_{j = s+1}^N \sigma_j ^2$. 

    Let $0< \epsilon \leq 1/2$, $s<r$ and $t_s = s + sr_s(X)$, then for $ s + \frac{7}{\epsilon^2} \ln(\frac{1}{\epsilon^2}) \leq k \leq t_s -1 $, the following holds:

    \[\mathbb{E}_{\operatorname{VS}} [\lVert X - \Pi_{S,k}X \rVert_{\operatorname{Fr}}^2] \leq (1 +2\epsilon)^2 \Phi_s(k) \lVert X - \Pi_{k}X \rVert_{\operatorname{Fr}}^2 \]

    where: $\Phi_s(k) = (1+ \frac{s}{k-s}) ÷ \sqrt{1 + \frac{2(k-s)}{t_s - k }}$

\end{proposition}

This upper bound is rather complicated to interpret in full generality. In their paper \cite{derezinski_improved_2021}, the authors show that it actually improves \ref{prop:VS} in cases where the singular values decay at a polynomial or an exponential rate.

\subsubsection{DPP sampling} 

The last algorithm we wanted to mention is taken from \cite{belhadji_determinantal_2020}. Their idea is to force the alignment between the subspace generated by the $k$ selected coordinates and the subspace generated by simultaneously force each principle angles to be the smallest possible.
In this regard, consider $S \subset [d]$, $\lvert S \rvert = k$ and denote $\mathcal{P}_S = \operatorname{Span} \{e_j | j \in S \}$. Let $\mathbb{\theta} = (\theta_i)_{1 \leq i \leq k}$ be the principal angle between $ \mathcal{P}_k$ and $\mathcal{P}_S$. Then it can be shown that: $\prod \limits_{i =1}^k \cos(\theta_i)^2 = \det(V_{S,[k]} V_{S,[k]}^T)$.
Hence sampling from an elementary DPP (of rank $k$) with kernel $K = V_k V_k^T$ emerges naturally from the geometry of the problem. This is what will be called DPP sampling in the following. 

Notice that this method actually generalizes leverage score importance sampling since the inclusion probabilities are given by: $\mathbb{P}_{\operatorname{DPP}}(i \in \mathcal{S}) = l_i^k $.
Belhadji et al. also show that there is a clear connection to volume sampling. Indeed, using the mixture property of DPP, one can notice that the highest mixture in the decomposition of $X^T X$ corresponds to the elementary DPP of kernel $V_k V_k^T$.

Without any further assumptions, they show the following: 

\begin{proposition} [\cite{belhadji_determinantal_2020} Proposition 16]
    Let $\mathcal{S}\sim \operatorname{DPP}(K)$, then:
\begin{equation}
     \mathbb{E}_{\operatorname{DPP}} [ \lVert X - \Pi_{S,k}X \rVert_{\operatorname{Fr}}^2] \leq k(d+1-k) \lVert X - \Pi_kX \rVert_{\operatorname{Fr}}^2
\end{equation}
\end{proposition}

Note that, in full generality, this bound is worse than the bound in the Volume sampling case. The main advantage of this method is that it has more flexibility.

Indeed, let the sparsity be the number $p$ of non zero $k$ leverage scores, and the flatness $\beta$ quantify the decay of the singular values of $X$ as: $\beta = \sigma_{k+1}^2 (\dfrac{1}{d-k}\sum \limits_{i = k+1} ^d \sigma_i^2)^{-1}$

Then the following holds:

\begin{proposition} [\cite{belhadji_determinantal_2020} Proposition 17]
    Let $\mathcal{S} \sim \operatorname{DPP}(K)$, then :
\begin{equation}
\mathbb{E}_{\operatorname{DPP}} [ \lVert X - \Pi_{S,k}X \rVert_{\operatorname{Fr}}^2] \leq (1+\beta \dfrac{p-k}{d-k}k) \lVert X - \Pi_kX \rVert_{\operatorname{Fr}}^2 
\end{equation}

\end{proposition}

The proof scheme for this result is quite similar to the one of \ref{prop:VS} with added technicalities. More precisely, notice that the inequality which was used there is rather loose, and a better understanding of the decay of the singular values of $X$ would give a better bound. Besides, this projection DPP has a strong connection to the principal angles as mentioned previously, and in a certain way, sampling from this DPP makes us find the columns that are best aligned with the principal eigenspaces. Using these facts, and correctly manipulating the symmetric polynomials lead to this result. We refer the reader to \cite{belhadji_determinantal_2020} for more details.

Notice that $ \beta \in [1,d-k]$, with the extreme cases being all singular values being equal after $k+1$ (which leads to $\beta =1$) and $\sigma_{k+1}$ being the last non zero singular value (which leads to $\beta = d-k$). 
In the regime of small $\beta$ and $p$, this bound is actually better than \ref{prop:VS} since $p\leq d$. Likewise, here to achieve a good rank $k$ approximation, only $k$ columns are sampled.

It is worth mentioning that it is possible to reduce the sparsity by looking at some "effective sparsity", and conditioning on sampling indices for which the leverage score contributes significantly to the whole. All of these results can also be extended when working with the spectral norm instead of the Frobenius one, see \cite{belhadji_determinantal_2020}.

%% file: nndpp.tex
\subsection{Background}

 Pruning neural networks consists in reducing the number of neurons in the network. Neural network pruning is a central question when it comes to optimizing neural network performances. Mariet and Sra developed a method, DIVNET in \cite{mariet_diversity_2017}, for  pruning neural network using DPP which has proven to be empirically efficient and is independent of the choice of other parameters such as the activation function or the number of hidden layers. We will describe the method for feed forward neural networks.

Their method can be defined in the following way. For a neural network with $M$ hidden layers trained on a database $\mathcal{T}$. For $1 \leq l \leq M$ assume the $l$-th layer has $n_l$ neurons. Let $1\leq i \leq n_l$ and $1 \leq j \leq n_{l+1}  $, the trainable weight from neuron $i$ in layer $l$ to neuron $j$ in layer $l+1$ is $W_{ij}^l$.  Define the activation vectors recursively :

\[
\begin{cases}
    v_1^0 = \mathcal{T} \\
    \forall 0 \leq l \leq M-1 \text{ , } \forall 1 \leq j \leq n_{l+1} \text{ , } v_j^{l+1} = \sigma ( \sum \limits_{i=1}^{n_l} W_{ij}^l v_i^l)
\end{cases}
\]

Assume  we want to prune the neural network on the $l$-th layer. Define then the kernel with respect to which the DPP is going to be sampled : 

\[ L'_l = ( \exp(-\beta \lVert v_i^l-v_j^l \rVert ^2))_{1\leq  i,j \leq n_l} \]

Here $\beta$ is a hyperparameter, which empirically is chosen to be of the magnitude $10/\lvert \mathcal{T} \rvert$. For better numerical stability, add a regularization term $\epsilon I$ so that the real kernel used is :

$$L_l = L_l' + \epsilon I $$ 

From here fix $k_n$ the number of neurons to keep in the layer and sample $\mathcal{S}$ a $k_n$-DPP with kernel $L_l$. For each $1 \leq j \leq n_{l+1}$, to improve the performances and minimize the information lost in the process, they introduce a reweighting phase which takes the following shape : 

\begin{equation} \label{eqn:reweighting}
\Tilde{\mathbf{W}_j^l} = \underset{\Tilde{\mathbf{W}_j^l} \in \mathbb{R}^{k}}{\operatorname{argmin}} \lVert \sum \limits_{i =1}^{n_l} W_{ij}^l v_i^l - \sum \limits_{i \in \mathcal{S}} \Tilde{W}_{ij}^l v_i^l \rVert 
\end{equation}

This method has proven to be particularly effective empirically while not requiring a retraining phase. Correctly understanding its theoretical implications is still an open question.

Looking at it through the lens of statistical mechanics offers some insight. This framework was already introduced in \cite{acharyya_statistical_2021}.
We aim at offering a comprehensive overview of the methods coming from statistical mechanics applied to the study of neural network pruning. We also extend and formalize the results from \cite{acharyya_statistical_2021} to more general settings. 

The analysis presented in this section was originally developed in the masters thesis \cite{petrovic2023pruning} of co-author V. Petrovic .

\subsection{The student/teacher framework}

The student/teacher framework is a common setting in statistical mechanical analysis of neural networks. The full theory behind this model is fully explained and detailed in \cite{engel_statistical_2001}. The model that we will use in this section comes from \cite{goldt_dynamics_2020} in the noiseless regime.

The framework is the following, assume we are given a data set $(X_\mu)_{1 \leq \mu \leq p}$ of i.i.d. Gaussian variables which take values in $\mathbb{R}^d$, of mean $0$ and variance matrix $I_d$. Let us get two neural networks respectively called the teacher and the student. The teacher network acts as a black box, to which we have access to the activation function and the number of layers and neurons per layer. It gets as inputs the data set $(X_\mu)_{1 \leq \mu \leq p}$ and assigns them outputs $(Y_\mu)_{1 \leq \mu \leq p}$, where if $1 \leq \mu \leq p$, $Y_\mu = \phi^*(X_\mu)$, where $\phi^*$ is a deterministic function. The student network is then another neural network, on which we have full control of the architecture, that we train to reproduce the results of the teacher network. In other words, the student network performs supervised learning with inputs $(X_\mu, Y_\mu)_{1 \leq \mu \leq p}$. Finally, let $g$ be the activation function, which is chosen to be the sigmoid function $\operatorname{erf}$. 
For the rest of the section, we will use $i,j,k$ to describe nodes from the student network and $n,m$ to describe nodes from the teacher network.

Here we will give special attention to the two layers student/teacher framework. In this case, both the student and the teacher framework have one hidden layer. Assume then, that the teacher network has $M$ hidden units and the student one has $K$ hidden units with $M \leq K$. The output of the teacher network is then:

\begin{equation*}
Y_\mu = \sum \limits_{m=1}^M v_m^* g(\dfrac{(w_m^{*})^T X_\mu}{\sqrt{N}}) = \phi^{*}(X_\mu) 
\end{equation*}  
The output of the student network is: 

\begin{equation*}
 \phi(X_\mu) = \sum \limits_{i=1}^K v_k g(\dfrac{w_k^T X_\mu}{\sqrt{N}})
\end{equation*}

The student network is trained to learn the teacher network on the empirical quadratic loss, defined as follows:

$$ L(\phi) = \dfrac{1}{2}\sum \limits_{\mu = 1}^p (\sum \limits_{i=1}^K v_k g(\dfrac{w_k^T X_\mu}{\sqrt{N}})-Y_\mu)^2$$

Define also the generalization error as: 

$$\epsilon_g (\phi) = \dfrac{1}{2}\braket{(\phi(X)- \phi^*(X))^2}$$

where $\braket{\cdot}$ is the average over the input distribution.

Finally introduce the following macro-parameters:

\[ Q_{ik} = \dfrac{w_i^Tw_k}{N} \text{ , }  T_{mn} = \dfrac{(w_m^*)^Tw_n^*}{N} \text{ , }  R_{in} =\dfrac{(w_n^*)^Tw_i}{N}  \]

The whole analysis will revolve around the dynamics established by Goldt et al. in \cite{goldt_dynamics_2020} .

\subsection{Dynamics of the macroparameters in online learning} \label{sec:dynamics}

Goldt et al establish very general results on the evolution of the paramaters in online learning. In the following, we will focus on a specific case mentioned by the authors of \cite{goldt_dynamics_2019} and further assume that:

\begin{itemize} \label{list:hypothesis}
    \item (H1) The dataset is large enough so that we visit each sample at most once during training.
    \item (H2) We are in the regime $N \rightarrow \infty$
\end{itemize}

In this context, $v$ satisfies:

\begin{equation} \label{eqn:evov}
\forall 1 \leq i \leq K \text{, } \dfrac{dv_i}{dt} = \eta_v[\sum \limits_{m=1}^M v_m^* I_2(i,m) - \sum \limits_{j =1}^K v_j I_2(i,j)] 
\end{equation}

where:

$$\forall 1 \leq i,k \leq K \text{, } I_2(i,k) = \braket{g(\lambda_i) g(\lambda_k)} =  \dfrac{1}{\pi} \arcsin{\dfrac{Q_{ik}}{\sqrt{1+Q_{ii}} \sqrt{1+Q_{kk}}}}$$

$$\forall 1 \leq i \leq K \text{, } \forall 1 \leq n \leq M \text{, } I_2(i,n) = \braket{g(\lambda_i)g(\rho_n)} = \dfrac{1}{\pi} \arcsin{\dfrac{R_{in}}{\sqrt{1+Q_{ii}} \sqrt{1+T_{nn}}}}$$

Let us also introduce the following quantity that plays a role in our coming analysis:

$$\forall 1 \leq m,n \leq M \text{, } I_2(m,n) = \braket{g(\rho_m)g(\rho_n)} = \dfrac{1}{\pi} \arcsin{\dfrac{T_{mn}}{\sqrt{1+T_{nn}} \sqrt{1+T_{mm}}}}$$

The authors of \cite{goldt_dynamics_2020} establish the following theorem:

\begin{theorem}
    With all the previous assumptions, the generalisation error $\epsilon_{g}$ satisfies: 

    \begin{equation} \label{eqn:generror}
 \epsilon_g(\phi) = \sum \limits_{i,k \in [K]} v_i v_k I_2(i,k) + \sum \limits_{n,m \in [M]} v_n^* v_m^* I_2(n,m) - 2 \sum \limits_{i \in [K], n \in [M]} v_i v_m^* I_2(i,n)
    \end{equation}
\end{theorem}

This formula is the key of our coming analysis of the effect of the pruning and reweighting phase using DPPs.

\subsection{Application to neural network pruning}
We assume that $T_{nm} = \delta_{nm}$ (as \cite{acharyya_statistical_2021}), i.e., that we work in the ideal case where the weights in the teacher network are a subfamily of orthonormal vectors. Assume further that both hypothesis (H1) and (H2) from section \ref{sec:dynamics} hold. 

 In this section, $\mathcal{S}$ describes the subset selected after pruning (i.e. the complementary of the pruned neurons) and $k_n = \lvert \mathcal{S} \rvert$. We also fix $\epsilon_{\operatorname{DPP}}$ (resp. $\hat{\epsilon}_{\operatorname{DPP}}$) the generalization error before reweighting (resp. after reweighting) of the pruned network.

The main goal of this section is to prove the following theorem:

\begin{theorem} \label{thm:nndpp}
    For any other sampling process $S$, with same marginal inclusion probabilities (ie $\mathbb{P}( i \in S) = \mathbb{P}_{\operatorname{DPP}}(i \in S)$ for all $1 \leq i \leq K$ ) the expectation (over the random sample) of the generalization error when the neural network is pruned and reweighted $\hat{\epsilon}_S$ satisfies:

    \[ \mathbb{E} [\hat{\epsilon}_S] \geq \mathbb{E}_{\operatorname{DPP}} [\hat{\epsilon}_{\operatorname{DPP}}] \]

\end{theorem}

This proposition asserts that in expectation, DPP sampling is the one that performs the best among any other random sampler in this specific framework.

\begin{remark}
    This is a generalization of \cite{acharyya_statistical_2021} in the following sense:
    \begin{itemize}
        \item The weights in the teacher network are not assumed to be all equal, likewise for the student network.
        \item The size of the subcells in the student network are not all equal. The number of neurons in the student network is not assumed to be a multiple of the number of neurons in the teacher network.
    \end{itemize}
\end{remark}

\subsection{Proof of the theorem}

The proof of this theorem can be decomposed in several parts.

\subsubsection{The student network achieves perfect reconstruction}

We first notice that in a specific context, the student network achieves perfect reconstruction. More precisely:

\begin{proposition} \label{prop:perfectrec}
    Let us write $\{1,\dots,K\} = \bigsqcup \limits_{m = 1}^M G_m $ where each $G_m$ is non empty and the $G_m$ are pairwise disjoint. This partition gives an equivalence relation on $[K]$, that we will write as $\sim$. Then, if:

    \begin{equation} \label{eqn:hyponw}
    \forall i \in \{1,\dots,K\} \text{, } \lVert w_i \rVert = \sqrt{N}
    \end{equation}

    \begin{equation} \label{eqn:hyponR}
    R_{in} = \begin{cases}
    1 & \text{ if } i \in G_n\\
    0 & \text{ otherwise } 
    \end{cases}
    \end{equation}

    we have $\phi \equiv \phi^*$ after training.
\end{proposition}

\begin{proof}

    First, notice that  $T_{nn} = 1$  means that $ \lVert w_n^* \rVert = \sqrt{N}$. Hence, if $i\in G_n$:

    $$ w_i^T w_n^* = \lVert w_i \rVert \lVert w_n^* \rVert = N R_{in}$$

    So we are in the equality case of the Cauchy-Schwartz inequality, $w_i$ and $w_n^*$ are colinear, of same norm and their scalar product is positive. This means that $w_i = w_n^*$. Observe that this implies straightforwardly, since $T_{nm} = \delta_{nm}$:

    \begin{equation} \label{eqn:hyponQ}
    Q_{ik} = \begin{cases}
    1 & \text{ if } i \sim k\\
    0 & \text{ otherwise } 
    \end{cases}
    \end{equation}

    Now, notice that, if $\mathbbm{1}$ is the indicator function, under equation \ref{eqn:hyponQ} and equation \ref{eqn:hyponR}: $I_2(i,k) = \frac{1}{6} \mathbbm{1}_{i \sim k}$ and  $I_2(i,n) = \frac{1}{6} \mathbbm{1}_{i \in G_n} $, where $I_2$ was introduced in the section \ref{sec:dynamics}.

    After training, the algorithm is assumed to reach a stable minimizer, which means $\frac{dv_i}{dt} = 0$ for all $i \in [K]$. Using equation \ref{eqn:evov}, we get that, after training, if $i \in G_n$:

    \begin{equation*} 
        \dfrac{dv_i}{dt} = 0 = \dfrac{\eta_v}{6}[v_n^* - \sum \limits_{i \in G_n} v_i]
    \end{equation*}

    Hence: 

    \begin{equation}\label{eqn:stablev}
        v_n^* = \sum \limits_{i \in G_n} v_i    
    \end{equation}

    Now, if $x\in \mathbb{R}^d$:

    \begin{eqnarray*}
     \phi(x) &= &\sum \limits_{i = 1}^K v_i g\Big (\frac{w_i^Tx}{\sqrt{N}} \Big ) =\sum \limits_{i = 1}^K v_i g \Big (\frac{w_i^Tx}{\sqrt{N}} \Big ) = \sum \limits_{m = 1}^M \sum \limits_{i \in G_m} v_i g \Big (\frac{w_i^Tx}{\sqrt{N}} \Big ) \\ 
     &= &\sum \limits_{m = 1}^M g \Big (\frac{(w_m^*)^Tx}{\sqrt{N}} \Big )\sum \limits_{i \in G_m} v_i  = \sum \limits_{m = 1}^M v_m^* g \Big (\frac{(w_m^*)^Tx}{\sqrt{N}} \Big ) = \phi^*(x)
    \end{eqnarray*}    
    \end{proof}

This establishes that the learning process achieves perfect reconstruction, when it reaches the fixed point corresponding to equation \ref{eqn:hyponQ} and equation \ref{eqn:hyponR}. Notice here that no assumptions were made on the size of the elements of the partition.

\subsubsection{Computing the generalization error}

The aim of this section is to find the generalization error, when we work in the framework of Proposition \ref{prop:perfectrec}. Here, we will not put much emphasis on the sampling process. Assume that some sampling process was used to produce $\mathcal{S} \subset [K]$, and that we apply some reweighting process to the second layer. Once again, assume the framework from section \ref{sec:dynamics} holds. Define then $q = \lvert \{ m | G_m \cap \mathcal{S} \neq \emptyset \} \rvert$, and, assume that $ \operatorname{EXP} = \{ m | G_m \cap \mathcal{S} \neq \emptyset \}$  are the explained neurons ($\operatorname{EXP}$ stands for explained).

\begin{theorem}\label{thm:geneerrorreweight}
    Let $(\tilde{v_i})_{i\in \mathcal{S}}$ be the reweighted weights for some reweighting procedure we apply to the second layer after pruning. Then:

    $$\hat{\epsilon}_g = \dfrac{1}{6} \sum \limits_{m \in \operatorname{EXP}} (v_m^* - \sum \limits_{i \in G_m \cap S} \Tilde{v}_i)^2 + \dfrac{1}{6} \sum \limits_{m \notin \operatorname{EXP}} (v_m^*)^2 $$
\end{theorem}

\begin{proof}
    It is a consequence of equation \ref{eqn:generror} applied to the pruned and reweighted network. Indeed, in this context, we have that $I_2$ is not changed (i.e. $I_2(i,k) = \frac{1}{6} \mathbbm{1}_{i \sim k}$, $I_2(i,n) = \frac{1}{6} \mathbbm{1}_{i \in G_n}$ and $I_2(n,m) = \frac{1}{6} \mathbbm{1}_{n=m}$) and:
    \begin{equation*}
     \hat{\epsilon}_g = \sum \limits_{i , k \in \mathcal{S}} \Tilde{v_i}\Tilde{v_k} I_2(i,k) + \sum \limits_{m , n \in [M]} v_n^*v_m^* I_2(n,m) +\sum \limits_{i \in \mathcal{S} \text{, } n \in [M]} \Tilde{v_i} v_n^* I_2(i,n)
     \end{equation*}

    Hence, thanks to our previous remark:

    \begin{eqnarray*}
     \hat{\epsilon}_g 
     &=  &\dfrac{1}{6}\Big [\sum \limits_{m \in \operatorname{ EXP }} (\sum \limits_{i \in G_m \cap \mathcal{S}}\Tilde{v}_i)^2 + \sum \limits_{m \notin \operatorname{ EXP 
 }} (v_m^*)^2 - 2 \sum \limits_{m \in \operatorname{ EXP }} v_m^* \sum \limits_{i \in G_m \cap \mathcal{S}} \Tilde{v}_i \Big ] \\
     &=  &\dfrac{1}{6} \Big [\sum \limits_{m \in \operatorname{ EXP }}(v_m^* - \sum \limits_{i \in G_m \cap \mathcal{S}}\Tilde{v}_i)^2 + \sum \limits_{m \notin \operatorname{ EXP }} (v_m^*)^2 \Big ]
     \end{eqnarray*}
\end{proof}

\subsubsection{DPP sampling}

In this section, the goal is to present and motivate the choice to sample  according to a DPP in this setting. From the framework we described, it is clear that, when pruned, the remaining network performs better when the remaining neurons are sampled from distinct groups $G_m$. This is why DPP sampling is here particularly interesting, since it enables us to sample $\mathcal{S}$ by maximizing the diversity within it.

To maximize the clarity of our statement, let us introduce the following quantities: for a sample $\mathcal{S} \subset [K]$, define $l_m = \lvert G_m \cap \mathcal{S} \rvert$, for $1 \leq m \leq M$. Notice that $\sum \limits _{m \in \operatorname{EXP}} l_m = k_n$, so that, since $l_m \geq 1$, $k_n \geq q$, the following property holds:

\begin{lemma} \label{lem:DPPsampling}
    Let $L = (L_{ij})_{i,j \in [K]}$ be the sampling kernel of a DPP $\mathcal{S}$. Assume $L_{ij} = f(Q_{ij})$ for some function well chosen $f$ ($L$ has to be semi-definite positive). Then $l_m \in \{0,1\}$ a.s. for every $m \in [M]$. 
\end{lemma}

\begin{proof}
    Notice that if $i,j \in G_m$ for some $m \in [M]$, then the vectors $Q_{i:} = (Q_{ik})_{k \in [K]}$ and and $Q_{j:} = (Q_{jk})_{k \in [K]}$ are equal. Hence, if $A \subset [K]$ is such that $i,j \in A$, the matrix $L_A$ has two identical columns so $\det(L_A) = 0$.

    Now:

    $$ \mathbb{P}( \mathcal{S} = A)  \propto \det (L_A) = 0$$

    This proves that the sampling process only outputs samples with elements from distinct groups $G_m$, hence $l_m \in \{0,1\}$ a.s.
\end{proof}

\begin{remark}
    Equivalently, if $i,j \in G_m$, $i \neq j$, then $\mathbb{P}(\{i \in \mathcal{S} \} \cap \{j \in \mathcal{S}\}) = 0$.
\end{remark}
We will give two examples of such kernels:

\begin{itemize}
    \item the first example is by taking $L_{ij} = Q_{ij}$. Notice that, since $Q_{ij} = \frac{1}{N} w_i^Tw_j$, $Q$ is a Gram-matrix so it is semi definite positive, and is an admissible choice for $L$.
    \item the second example is the following: $L_{ij} = \exp(-\dfrac{\beta}{N} \lVert w_i -w_j \rVert^2) $ for some $\beta>0$. Notice that $L_{ij} = \exp(-2\beta (1+Q_{ij})) $ so the previous lemma applies.
\end{itemize}

\subsubsection{Analysis in the pruned and reweighted case}

It is time to come to analyse the generalization error, when we prune and reweight the network. Notice that theorem \ref{thm:geneerrorreweight} gives a natural way to reweight our neural network with respect to the generalization error. Indeed, to minimize the true generalization error, it is sufficient to choose the $\Tilde{v}_i $ in the following way:

\begin{equation}\label{eqn:choicereweight}
\forall m \in \operatorname{EXP} \text{, } \sum \limits_{i \in G_m \cap \mathcal{S}} \Tilde{v}_i = v_m^{*} = \sum \limits_{i \in G_m}v_i
\end{equation}

The second part of this equation was added to remind that this reweighting process can be performed even though we do not have prior access to $v_m^{*}$.

Finally the proof of theorem \ref{thm:nndpp} is as follows:

\begin{proof}
    This a consequence of theorem \ref{thm:geneerrorreweight} by plugging equation \ref{eqn:choicereweight} into the formula. Indeed, with this choice of weights, we get that:
    $$ \hat{\epsilon}_S = \sum \limits_{m \notin \operatorname{EXP}} (v_m^*)^2$$

    Hence :

    $$ \mathbb{E}[\hat{\epsilon}_S] = \sum \limits_{m=1}^M (v_m^*)^2 \mathbb{P}(m \notin \operatorname{EXP})$$

    Look at $\mathbb{P}(m \notin \operatorname{EXP})$ for $m \in [M]$:

    $$\mathbb{P}(m \notin \operatorname{EXP}) = 1 - \mathbb{P}(m \in \operatorname{EXP}) = 1 - \mathbb{P}(\exists i \in G_m \cap S) = 1 - \mathbb{P}(\bigcup \limits_{i \in G_m} \{ i \in S \}) $$

    Using the union bound:

    $$\mathbb{P}(m \notin \operatorname{EXP}) \geq 1 - \sum \limits_{i \in G_m} \mathbb{P}(i \in \mathcal{S})$$
    
    Now, thanks to the previous section, we know that, if $i$ and $j$ lie in the same $G_m$, for DPP sampling, $\mathbb{P}_{\operatorname{DPP}}(\{i \in S\} \cap \{j \in S \}) = 0 $. Hence:

    $$ \mathbb{P}_{\operatorname{DPP}} (\bigcup \limits_{i \in G_m} \{ i \in S \})) = \sum \limits_{i \in G_m} \mathbb{P}_{\operatorname{DPP}}(i \in \mathcal{S})$$

    This proves that for all $m \in [M]$:

    $$ \mathbb{P}(m \notin \operatorname{EXP}) \geq \mathbb{P}_{\operatorname{DPP}}(m \notin \operatorname{EXP})$$

\end{proof}

%% file: quantum.tex
At the end of the 50s, Hanbury-Brown and Twiss showed evidence that, when a detector of single photons is placed in the electric field generated by a source of thermal light, the detection times tend to lump together; this is the so-called \emph{bunching effect}.
In what physicists call the \emph{semi-classical} picture, photon detection times are modeled by an inhomogeneous Poisson point process of intensity function given by the electric field, and the assumption of thermal light corresponds to taking this intensity function to be a realization of a Gaussian process with smooth samples. 
In modern parlance, this hierarchical model makes the detection times a Cox point process. 
The smoothness of the samples from the Gaussian process explains why the detection times appear to be lumped together.

While her thesis initially was concerned with finding a stochastic description of this bunching effect, without the \emph{semiclassical} description of the electric field as a function-valued stochastic process, Odile Macchi then turned to model a similar situation where photons are replaced by fermions, like electrons. 
There, she found out that detection events are negatively correlated, evidencing an \emph{anti-bunching} effect that was later experimentally observed.
Macchi formalized DPPs precisely to model that detection process \cite{Mac72}; see also \cite{BFBDHRRSW22Sub} for a modern review. 
Mathematically, DPPs arise from the combined action of Born's rule --the quantum physics axiom that lab measurements are draws from a random variable parametrized by two operators on a Hilbert space-- and the canonical anti-commutation relations satisfied by the building blocks of operators describing fermions.

In this section, we walk in Macchi's footprints and illustrate how DPPs are related to quantum fermionic models in a simple setting, showing in passing that we can design an algorithm to sample a given DPP on a quantum computer.
This section paraphrases \cite{BaFaFe24}, to which we refer the reader for more details and important references, such as \cite{JSKSB18, KePr22}.

\subsection{The formalism of quantum physics}
A minimalistic physical model is $(i)$ a correspondence between a set of physical notions and a set of mathematical objects, as well as $(ii)$ a law to connect these mathematical objects to a measurement.
In quantum physics, the mathematical objects to describe an experiment are a Hilbert space $(\mathbb{H},\braket{\cdot}{\cdot})$ and a collection of self-adjoint operators from $\mathbb{H}$ to $\mathbb{H}$ called \emph{observables}.
A state of the physical system, labeled as some string of characters $\psi$, is described by the projection onto the line $\mathbb{C}\ket{\psi}$ in $\mathbb{H}$. 
Here $\ket{\psi}$ is a vector of $\mathbb{H}$ of unit norm that we associate to the label $\psi$. 
Writing vectors in the form $\ket{\psi}$ (a \emph{ket}) and linear forms as $\bra{\psi}:\ket{\phi} \mapsto \braket{\psi}{\phi}$ (a \emph{bra}) is a useful notational trick known as the \emph{bra-ket} notation.
It is compatible with our denoting the inner product as $\braket{\cdot}{\cdot}$.
In particular, the projection onto $\mathbb{C}\ket{\psi}\subset\mathbb{H}$ simply writes $\ketbra{\psi}$, because
$$
\ketbra{\psi} \ket{\phi} = \braket{\psi}{\phi} \ket{\psi}.
$$

For simplicity, we henceforth restrict to the $2^N$-dimensional example $\mathbb{H} = (\mathbb{C}^2)^{\otimes N}$, for some $N\geq 0$. 
This is the Hilbert space that is used in quantum computing to describe a system of $N$ qubits, i.e. $N$ particles that can each assume two states.
In this finite-dimensional setting, an observable $A$ is simply a self-adjoint matrix, so that it diagonalizes in an orthonormal basis $A = \sum_{i=1}^{M} \lambda_i u_i u_i^*$, where we let $M=2^N$.
Observables are meant to correspond to physical properties that one can measure.
The final ingredient of the model is \emph{Born's rule}, which associates a state-observable pair to a random measurement.
More precisely, Born's rule stipulates that measuring the physical property described by $A$ in the state described by $\ketbra{\psi}$ is equivalent to drawing from the random variable $X$ supported on the spectrum $\Lambda = \{\lambda_1, \dots, \lambda_{M}\} \subset \mathbb{R}$ of $A$, and described by 
\begin{equation}
    \label{e:born}
    \mathbb{E} h(X) = \mathrm{Tr}\left(\ketbra{\psi} \sum_{i=1}^M h(\lambda_i)u_i u_i^*\right ) = \sum_{i=1}^M h(\lambda_i) \vert\braket{\psi}{u_i}\vert^2,
\end{equation}
where $h$ is any real-valued function on $\Lambda$.
Furthermore, independently preparing $\ket{\psi}$ and measuring $A$ several times in a row yields statistically independent draws of the random variable $X$.

\subsection{The canonical anticommutation relations lead to DPPs}

In the rest of this section, we describe a state $\ket{\psi}$ and an observable $A$ on $\mathbb{H} = (\mathbb{C}^2)^{\otimes N}$ that can be efficiently prepared and measured on a quantum computer, and such that $X$ in \eqref{e:born} is a user-defined projection DPP on $\{1, \dots, N\}$.
Generic DPPs on finite ground sets with Hermitian kernels, and more generally Pfaffian point processes, can also be obtained \cite{BaFaFe24}, but we focus here on projection DPPs for simplicity.

Let $c_1, \dots, c_N$ be any operators on $\mathbb{H}$ that satisfy the canonical anticommutation relations
\begin{equation}
        \label{e:CAR}
        \lbrace c_i,c_j \rbrace = \lbrace c^*_i,c^*_j \rbrace  = 0
        \qquad \text{and} \qquad
        \lbrace c_i,c^*_j \rbrace = \delta_{ij} \mathbb{I},
\end{equation}
where $\{u,v\} \triangleq uv+vu$ is the anti-commutator of operators $u$ and $v$.   
Assuming the existence of a set of operators satisfying \eqref{e:CAR} for the moment, one can show \cite{Nie05} the existence of a simultaneous eigenvector to all $c_i^*c_i$, which we denote by $\ket{\emptyset}$, always with eigenvalue zero. Moreover, the vectors 
$$
    \ket{\mathbf{n}} = \prod_{i=1}^N (c_i^*)^{n_i}\ket{\emptyset}, \quad \mathbf{n}\in\{0,1\}^N, 
$$
form a basis of $\mathbb{H}$, such that 
$$
    c_i^* c_i \ket{\mathbf{n}} = n_i\ket{\mathbf{n}}.
$$
This basis is called the \emph{Fock basis} corresponding to $c_1, \dots, c_N$. 

Coming back to the existence of operators satisfying \eqref{e:CAR}, an example is given by the Jordan-Wigner operators $a_1, \dots, a_N$; see e.g. Definition 3.3 in \cite{BaFaFe24}. 
The Jordan-Wigner operators come with additional tractability, in the sense that any product of $a_i$s applied to $\ket{\emptyset}$ is easy to prepare on a quantum computer, and any observable of the form $a_i^*a_i$ is easy to measure in such a state. 

We are ready to define the state-observable pair that will correspond to a given projection DPP. Let $ V = ((v_{ij}))$ be an $N\times N$ unitary matrix, and consider 
$$
    b_k = \sum_{j=1}^N v_{kj}^* a_j, \quad 1\leq k \leq N.
$$
One can show that $b_1, \dots, b_k$ also satisfy \eqref{e:CAR}. 
Now, consider the state $\ketbra{\psi}$, where
$$
    \ket{\psi} = b_1^* \dots b_r^* \ket{\emptyset},
$$
and consider the observable $A = \prod_{i=1}^N a_i^* a_i$.
Since all $a_i^*a_i$ are simultaneously diagonalizable in the Fock basis, they commute. 
In particular, $A$ is self-adjoint, making it a \emph{bona fide} observable. 
Decomposing the trace appearing in Born's rule \eqref{e:born} in the Fock basis, and carefully applying the canonical anti-commutation relations, one obtains a random variable $X$ that is a projection DPP with kernel $K = V_{[r],:}^*V_{[r],:}$.
The determinantal character of the distribution precisely comes from this manipulation of anticommuting operators in Born's rule, which physicists have come to call \emph{Wick's theorem}.
We refer to Appendix A.7 of \cite{BaFaFe24} for a detailed computation, with the same notations as we used here.

\begin{figure}[h]
    \centering
    \includegraphics[width=1\linewidth]{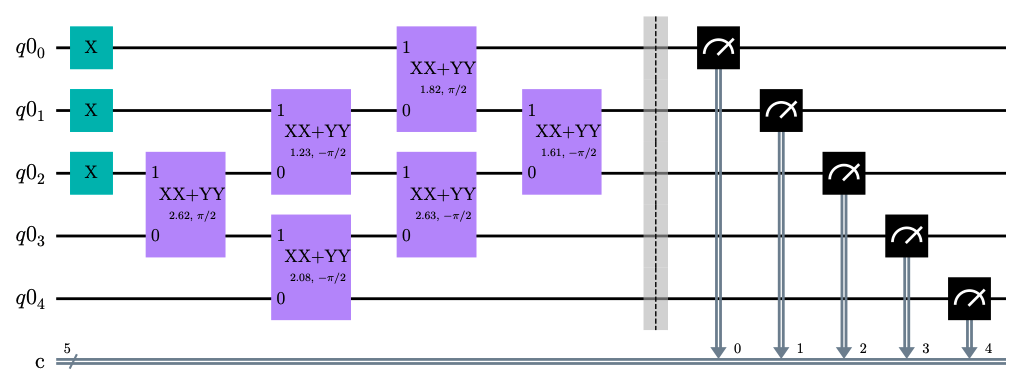}
    \caption{A circuit sampling a DPP with projection kernel of rank $r = 3$, with $N = 5$ items. On the left-hand side, Pauli $X$ gates are used to create fermionic modes on the three first qubits. Note also the parallel QR Givens rotations on
neighbouring qubits indicated by parametrized $XX + YY$ gates. On the right-hand side, measurements of occupation numbers are denoted by black squares (c.f. \cite{BaFaFe24}).}
    \label{fig:quantum}
\end{figure}

In addition to a more extensive introduction to the quantum computing tools used here, the paper \cite{BaFaFe24} also contains a derivation of state-observable pairs corresponding to generic DPPs with Hermitian kernels and Pfaffian point processes, as well as experimental demonstrations on IBM quantum computers.